%% file: main.tex
\documentclass[conference]{IEEEtran}

\input{setting}
\usepackage{cite}
\usepackage{amsmath,amssymb,amsfonts}
\usepackage{algorithmic}
\usepackage{graphicx}
\usepackage{textcomp}
\usepackage{xcolor}
\def\BibTeX{{\rm B\kern-.05em{\sc i\kern-.025em b}\kern-.08em
    T\kern-.1667em\lower.7ex\hbox{E}\kern-.125emX}}
\begin{document}

\title{\FairRep: Scalable and Effective Data Pre-Processing for Causal Fairness}

\author{
\IEEEauthorblockN{Ying Zheng}
\IEEEauthorblockA{\textit{National University of Singapore} \\
zheng.ying@u.nus.edu}
\and
\IEEEauthorblockN{Yangfan Jiang}
\IEEEauthorblockA{\textit{National University of Singapore} \\
jyangfan@u.nus.edu}
\and
\IEEEauthorblockN{Kian-Lee Tan}
\IEEEauthorblockA{\textit{National University of Singapore} \\
tankl@comp.nus.edu.sg}
}

\maketitle

\input{section/0_abstract}

\input{section/1_introduction}

\input{section/2_preliminary}

\input{section/3_problem}

\input{section/4_methodology}

\input{section/5_evaluation}

\balance
\input{section/6_relatedwork}

\input{section/7_conclusion}

\clearpage

\bibliographystyle{IEEEtran}
\bibliography{main}

\appendix
\input{section/appendix}

\end{document}

%% file: setting.tex
\usepackage{balance}

\usepackage[unicode]{hyperref}

\hypersetup{
	pdfstartview=FitH,
	bookmarksnumbered=true,
	bookmarksopen=true,
	colorlinks,
	linkcolor={violet!90!black},
	citecolor={blue!60!black},
	urlcolor={blue!80!black}
}

\usepackage[ruled, vlined, linesnumbered]{algorithm2e}
\SetKw{Return}{return}
\usepackage{multirow}
\usepackage{booktabs}
\usepackage{array}
\usepackage{xcolor}
\usepackage{subcaption}
\usepackage{enumitem}
\usepackage{float}

\usepackage{mathrsfs}
\usepackage{amsmath}
\usepackage{amsthm}

\newtheorem{definition}{Definition}

\newtheorem{theorem}{Theorem}
\newtheorem{corollary}{Corollary}
\newtheorem{proposition}{Proposition}

\allowdisplaybreaks
\newcolumntype{L}[1]{>{\raggedright\arraybackslash}p{#1}}

\newcommand{\ceil}[1]{\left\lceil #1 \right\rceil}

\newcommand{\FairRep}{$\texttt{CausalPre}$}
\newcommand{\FairRepAd}{\texttt{CausalPre+}}

\usepackage{soul}

%% file: section/0_abstract.tex
\begin{abstract}

Causal fairness in databases is crucial to preventing biased and inaccurate outcomes in downstream tasks. While most prior work assumes a known causal model, recent efforts relax this assumption by enforcing additional constraints. However, these approaches often fail to capture broader attribute relationships that are critical to maintaining utility. This raises a fundamental question: \emph{Can we harness the benefits of causal reasoning to design efficient and effective fairness solutions without relying on strong assumptions about the underlying causal model?} 

In this paper, we seek to answer this question by introducing \FairRep, a scalable and effective causality-guided data pre-processing framework that guarantees \emph{justifiable fairness}, a strong causal notion of fairness. \FairRep\ extracts causally fair relationships by reformulating the originally complex and computationally infeasible extraction task into a tailored distribution estimation problem. To ensure scalability, \FairRep\ adopts a carefully crafted variant of low-dimensional marginal factorization to approximate the joint distribution, complemented by a heuristic algorithm that efficiently tackles the associated computational challenge. Extensive experiments on benchmark datasets demonstrate that \FairRep\ is both effective and scalable, challenging the conventional belief that achieving causal fairness requires trading off relationship coverage for relaxed model assumptions.

\end{abstract}

%% file: section/1_introduction.tex
\section{Introduction} \label{sec: intro}

Machine learning (ML) systems are increasingly integrated into decision-making processes in domains such as education~\cite{admission1975}, finance~\cite{finance2007}, employment~\cite{amazon2022}, advertising~\cite{recommend2022}, and law enforcement~\cite{policing2017, crime2009}. While these systems offer efficiency and scalability, they also pose serious concerns about fairness~\cite{lin2024mitigating, zhang2025efficient, yang2024noninvasive, zheng2024fairgen, pradhan2022interpretable, surve2025explaining, tsioutsiouliklis2021fairness, tae2024falcon}. In particular, their reliance on historical data can unintentionally amplify biases, producing inaccurate, discriminatory outcomes with severe real-world impacts in high-stakes areas like criminal justice.

These concerns have motivated the development of fairness-aware data pre-processing techniques within database management systems (DBMS)~\cite{pirhadi2024otclean, islam2022through, nargesian2021tailoring, salimi2019interventional, lahoti2019operationalizing, shahbazi2024fairness, zhou2025intervention, azzalini2022efairdb}. Compared to traditional fairness interventions at the model training or inference stages~\cite{mehrabi2021asurvey, jiang2019wasserstein, cong2024fairsample, agarwal2018areductions, chowdhury2024enhancing, zafar2017fairness}, pre-processing methods offer: (i)~a once-for-all benefit, meaning that once data is calibrated for fairness, it can be used in any downstream task, regardless of the ML model employed; and (ii)~a user-friendly workflow, as fairness considerations are directly embedded into the data pre-processing pipeline, enabling practitioners to focus on the downstream task without specialized fairness expertise. 

A straightforward approach to achieve this is to remove all sensitive attributes (e.g., gender and race) from the training data. However, such ad hoc solutions often fail in practice, as non-sensitive attributes may act as proxies for sensitive ones, particularly when strong correlations exist~\cite{kusner2017counterfactual, salimi2019interventional}. This implicit leakage of sensitive information creates further challenges for DBMS, making it difficult to trace or diagnose the sources of discrimination in downstream tasks~\cite{salimi2019interventional}. 

To address this issue, researchers have developed rigorous and principled definitions of fairness~\cite{pessach2023areview, mehrabi2021asurvey, islam2022through}, with a focus on understanding how sensitive attributes influence decision-making processes. Among them, \emph{causal fairness}~\cite{salimi2019interventional,pirhadi2024otclean,pujol2023prefair} has gained traction by explicitly modeling causal pathways through which biases may propagate. 

\begin{figure*}
    \centering
    \includegraphics[width=0.85\linewidth]{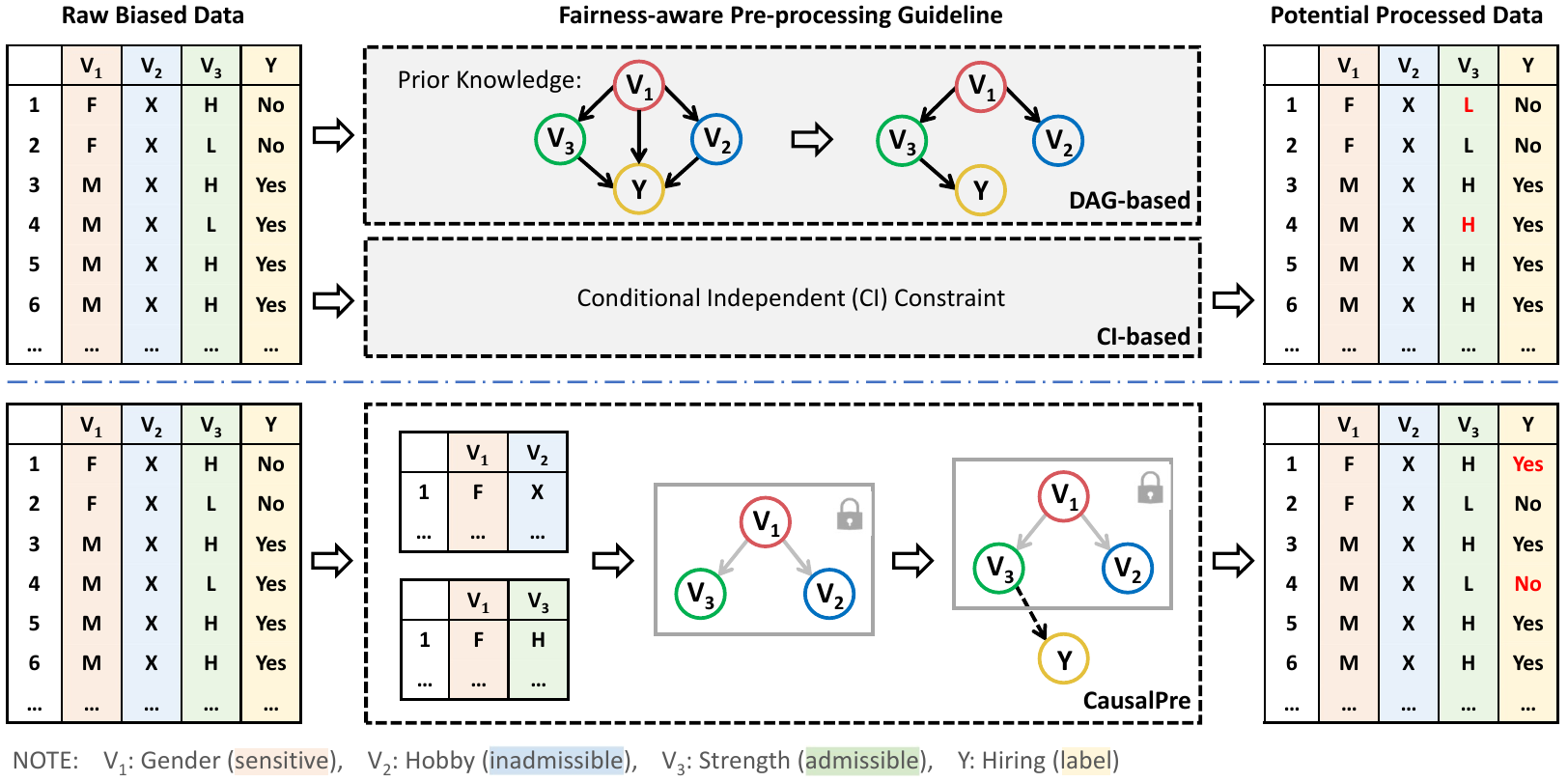}
    \caption{High-level comparison of fairness-aware pre-processing strategies on the biased manual labor hiring dataset. The top-half of the figure shows existing DAG-based or CI-based schemes, while the bottom-half illustrates {\FairRep}.}
    \label{fig: core_idea}
\end{figure*}

\subsection{Motivation}\label{subsec: intro-motivation}

Existing causal fairness-aware solutions~\cite{chiappa2019path, kilbertus2017avoiding, li2025local, zuo2024interventional, xu2019achieving} typically rely on a predefined causal structure (e.g., the graph in Figure~\ref{fig: core_idea}'s prior knowledge) to guide fairness interventions. However, such structures are data-specific and often unknown in advance. Moreover, building a causal graph from scratch, e.g., using causal discovery methods \cite{chickering1995learning,neapolitan2004learning,glymour2019review,pearl2009causality}, can be computationally intensive. This limitation is especially problematic for applications (e.g., analyzing user behavior in niche mobile apps), where fairness is a critical concern but explicit causal relationships among attributes are difficult to specify. 

To address this, recent efforts have sought to relax the reliance on predefined causal structures by enforcing additional constraints. One of such causal fairness notions, \emph{justifiable fairness}~\cite{salimi2019interventional}, requires that the effect of sensitive attributes on predictions be mediated only through {admissible} attributes, which are deemed relevant and fair for decision-making. To guarantee justifiable fairness, it is sufficient to pre-process the dataset with conditional independence (CI) constraints \cite{salimi2019interventional,pirhadi2024otclean}. However, a major issue of these approaches is their \emph{exclusive} focus on enforcing CI constraints, often at the expense of preserving data utility. In particular, optimizing solely for CI enforcement may significantly disrupt the underlying statistical properties of the input data, potentially leading to out-of-distribution issues in downstream tasks. In addition, these methods typically neglect the preservation of intra-record attribute relationships, resulting in nonsensical or implausible record instantiations that undermine data utility. Moreover, it is computationally expensive to enforce CI constraints, especially for datasets with rich attribute sets and large domain sizes, which limits the scalability of these solutions.

\subsection{Our Contributions} \label{subsec: intro-contri}

In this paper, we present {\FairRep}, a scalable and effective data pre-processing framework for algorithmic fairness. \FairRep\ offers formal fairness guarantees while scaling to datasets with complex attribute spaces, without any predefined causal structure. This is a particularly important advantage in real-world scenarios, where datasets often contain rich and diverse attributes and the underlying causal models are rarely known or available. To achieve this, {\FairRep} introduces a new design paradigm that differs fundamentally from previous approaches: instead of explicitly constructing full causal graphs or solely enforcing CI constraints, {\FairRep} infers causally fair relationships directly from the data and uses these to sanitize the dataset, thereby ensuring fairness while preserving the general statistical properties of the original data.

\vspace{2mm} \noindent\textbf{Challenges.} 
\FairRep\ focuses on tackling two challenges inherent in developing fairness-aware pre-processing techniques without pre-defined causal structure, while still preserving statistical fidelity. First, causal fairness notions are inherently defined and interpreted with respect to a causal graph. Without an underlying graph, it becomes non-trivial to infer causally fair relationships among attributes. Yet, such relationships are essential for ensuring and validating fairness guarantees. Second, the goal of a fairness-aware data pre-processing scheme is to ensure fairness without compromising utility. However, overly modifying data can harm utility, while preserving utility may undermine fairness. Striking this balance, especially at scale, poses major design challenges.

\vspace{2mm} \noindent\textbf{Solution overview.}
To address these challenges, we first conduct an in-depth analysis of the objective of causal-fairness-aware pre-processing under utility considerations. Our analysis yields a theoretical refinement that reinterprets the goal of fairness-aware data pre-processing as a tailored statistical property estimation problem that implicitly encodes fairness requirements. 
Specifically, the task seeks to approximate the data-generating distribution in a hypothetical fair world, where sensitive attributes influence outcomes only through admissible pathways. This formulation provides a principled alternative to prior approaches by bypassing the need for an explicit causal graph, eliminating causal structural details that are irrelevant to the fairness objective.

While this refinement significantly reduces the complexity of designing causal fairness solutions, practical instantiation remains challenging. On the one hand, directly estimating the full joint distribution is computationally infeasible due to its exponential dependence on dimensionality. On the other hand, naive marginal-based approximations may overlook important inter-attribute relationships, resulting in implausible records and degraded downstream task utility. To address this, we develop a carefully crafted variant of the low-dimensional marginal-based approximation method that decomposes the high-dimensional joint distribution into a series of smaller, more manageable marginals, each defined over a carefully selected subset of attributes. Specifically, these subsets are chosen to capture high mutual information while satisfying structural constraints such as size limits and overlap requirements. We formalize the marginal selection process as a constrained clique generation problem on a weighted complete graph, with the objective of identifying cliques that jointly maximize intra-clique dependencies under given constraints.

A practical concern in implementing this process is the computational hardness of the clique generation problem, which renders exact solutions infeasible at scale. To overcome this limitation, we further design a heuristic that efficiently produces high-quality clique sets, enabling scalable and effective approximation of the fair data distribution. By integrating these technical components, {\FairRep} efficiently pre-processes the input dataset to ensure that its empirical distribution aligns with the target fair distribution. 

We conduct extensive experiments on multiple benchmarks using various ML models, including logistic regression, random forest, and neural networks. The results demonstrate that {\FairRep} achieves strong causal fairness while preserving high utility across all evaluated tasks.

\vspace{2mm} \noindent\textbf{Motivating example.}
To further illustrate the intuition behind the core idea and effect of \FairRep, consider a manual labor hiring dataset in Figure~\ref{fig: core_idea}, which reveals a biased hiring pattern: males are more likely to be hired than females, regardless of their physical strength. Unlike previous solutions that either rely on prior knowledge of causal structure or exclusively enforce CI constraints, \FairRep\ extracts causally fair relationships directly from data to guide pre-processing. Specifically, it begins by decomposing the dataset column-wise into carefully chosen low-dimensional subsets, and then uses these subsets to approximate the joint relationship that should remain unchanged during processing. Building on this, \FairRep\ then identifies fairness-aware decision-making factors, enabling the recovery of complete causally fair relationships.  These relationships act as if a fair attribute graph were available, providing principled guidance for fairness-aware processing.

The rightmost part of Figure~\ref{fig: core_idea} illustrates the resulting modifications, with altered entries highlighted in red. The upper example, which is typically generated by CI-based pre-processing methods, distorts the statistical relationship between gender and strength, misleadingly implying that all females have low strength and all males have high strength. In contrast, the lower example produced by \FairRep\ avoids such distortions: it preserves the realistic variability of strength across genders while correcting the hiring bias.

%% file: section/2_preliminary.tex
\section{Preliminaries} \label{sec: pre}

This section first introduces the concept of causal directed acyclic graphs (DAGs), and then presents the fairness notions and information measures used in this paper.

\subsection{Causal DAGs} \label{subsec: preliminary-dag}

Given attributes $\mathcal{V}=\{V_1,\dots,V_d\}$ and directed edges $\mathcal{E}\subseteq \mathcal{V}\times \mathcal{V}$, a causal DAG is $\mathcal{G}=(\mathcal{V},\mathcal{E})$, where each $(V_i,V_j)\in\mathcal{E}$ denotes a direct causal influence from $V_i$ to $V_j$ and the graph contains no directed cycles.

Causal DAGs provide the foundation for causal fairness, which is defined through the notion of \textit{interventions} and formally expressed by the \textit{do}-operator~\cite{pearl2009causality}. An intervention $\text{\textit{do}}(X{=}x)$ enforces $X$ to take value $x$ by removing all incoming edges to $X$, and the resulting distribution $\mathbb{P}[O \mid \text{\textit{do}}(X{=}x)]$ captures the causal effect of $X$ on $O$.

Beyond defining causal fairness notions, causal DAGs also provide unique advantages for modeling and decomposing the joint distribution $\mathbb{P}[\mathcal{V}]$. Let $\Pi_i$ denote the parent set of $V_i$. According to the \emph{local Markov property}~\cite{pearl2009causality}, each attribute $V_i$ is conditionally independent of its non-descendants given its parents $\Pi_i$, and thus the joint distribution factorizes as
\begin{align} \label{eq: bayes}
    \mathbb{P[\mathcal{V}]} = \mathbb{P}[V_1, \dots, V_d] = \prod_{i=1}^{d} \mathbb{P}[V_i \mid \Pi_i].
\end{align}

Finally, dependencies in a DAG are characterized by \emph{d-separation}~\cite{pearl2009causality}. Two sets of attributes $X$ and $Y$ are d-separated by $Z$ if every path between $X$ and $Y$ is blocked by $Z$; in this case, the conditional independence $\left(X \perp \!\!\! \perp_d Y \mid Z\right)$ holds. A distribution $\mathbb{P}$ is said to be \textit{Markov compatible} with a DAG $\mathcal{G}$ if every d-separation in $\mathcal{G}$ implies a corresponding conditional independence in $\mathbb{P}$, and it is \textit{faithful} if the converse holds. As in previous work on causal fairness~\cite{salimi2019interventional, galhotra2022causal, pujol2023prefair}, we adopt the general assumption that the dataset distribution is Markov compatible with and faithful to the underlying attribute graph.

\subsection{Causal Fairness} \label{subsec: preliminary-fairness}

Let $\mathcal{M}: \text{Dom}(\mathcal{X})\mapsto\text{Dom}(O)$ denote the classifier that maps input features $\mathcal{X}$ to an outcome $O$, where Dom($\cdot$) denotes the domain. Let $\mathcal{S}\subseteq\mathcal{X}$ denote the set of sensitive attributes. We now introduce the notion of causal fairness as follows.

\begin{definition}[$\mathcal{K}$-fair~\cite{salimi2019interventional}] \label{def: kfair}
    Given a subset $\mathcal{K} \subseteq \mathcal{X} \setminus\{\mathcal{S}\}$, we say that a classifier $\mathcal{M}:\text{Dom}(\mathcal{X})\mapsto\text{Dom}(O)$ is \emph{$\mathcal{K}$-fair} with respect to sensitive attributes $\mathcal{S}$ if, for any instantiation $\mathcal{K} = \kappa$, the following holds:
    \begin{align} \label{eq: fair_def}
        &\mathbb{P}[O = o \mid \text{do}(\mathcal{S} = \text{\scriptsize$\mathcal{S}$}_0), \text{do}(\mathcal{K} = \text{\scriptsize$\mathcal{K}$})] \notag \\
        = &\mathbb{P}[O = o \mid \text{do}(\mathcal{S} = \text{\scriptsize$\mathcal{S}$}_1), \text{do}(\mathcal{K} = \text{\scriptsize$\mathcal{K}$})].
    \end{align}
\end{definition}

\begin{definition}[Justifiable Fairness~\cite{salimi2019interventional}] \label{def: justifiable}
    Let $\mathcal{A} \subseteq \mathcal{X}$ denote the set of \emph{admissible attributes} that are allowed to influence the outcome despite their causal link to sensitive attributes. A classifier $\mathcal{M}$ is said to be \textit{justifiably fair} if it is $\mathcal{K}$-fair for every superset $\mathcal{K}$ such that $\mathcal{A} \subseteq \mathcal{K} \subseteq \mathcal{X}$.
\end{definition}

Note that justifiable fairness is defined in terms of the classifier's outcome. To distinguish the effect of data pre-processing from model training, we follow prior work~\cite{salimi2019interventional, galhotra2022causal} and assume that the classifier is \emph{reasonable}, i.e., it closely approximates the underlying data distribution it was trained on. Specifically, given a dataset with attributes $\mathcal{X}\cup\{Y\}$, where $Y$ denotes the ground truth label, a classifier is reasonable if $\mathbb{P}[Y=y\mid\mathcal{X}=x] \approx \mathbb{P}[O=y\mid\mathcal{X}=x], \forall y\in\text{Dom}(O)$. Hereafter, we slightly abuse the notation and use the outcome variable $O$ and label variable $Y$ interchangeably. 

To further illustrate the concept of justifiable fairness, consider the two DAGs in the ``DAG-based'' block of Figure~\ref{fig: core_idea}. Suppose we have a reasonable classifier represented by the mapping rule in the left DAG. In this setting, $\mathcal{S}{=}\{V_1\}$, $\mathcal{A}{=}\{V_3\}$, $O{=}Y$, and $\mathcal{K}{=}\{V_3\}$ or $\{V_2, V_3\}$. Now consider the case $\mathcal{K}{=}\{V_3\}$. Applying the intervention $\text{do}(\mathcal{K} {=} \text{\scriptsize$\mathcal{K}$})$ removes the edges incoming to $\mathcal{K}$, specifically $V_1 {\rightarrow} V_3$. Even after this intervention, there remains an active causal path (i.e., $V_1 {\rightarrow} V_2 {\rightarrow} Y$) from $\mathcal{S}$ to $O$. As a result, the causal effect of the sensitive attribute on the decision persists, Equation~(\ref{eq: fair_def}) is not satisfied, and $\mathcal{K}$-fairness is violated; thus, the classifier is justifiably unfair. In contrast, if both paths $V_1 {\rightarrow} Y$ and $V_1 {\rightarrow} V_2 {\rightarrow} Y$ are blocked, as shown in the right DAG, then for any choice of $\mathcal{K}$, there exists no causal path and therefore no causal effect from $\mathcal{S}$ to $O$. In this case, the classifier is $\mathcal{K}$-fair for all possible $\mathcal{K}$ and is therefore justifiably fair.

Since the \textit{do}-operator is defined over a causal DAG, justifiable fairness admits an intuitive graphical interpretation:
\begin{theorem}[\cite{salimi2019interventional}] \label{theorem: justifiable}
    Given a causal DAG $\mathcal{G}$ over attributes $\mathcal{V}$, if every directed path from any sensitive attribute in $\mathcal{S}$ to the outcome attribute $O$ contains at least one admissible attribute in $\mathcal{A}$, then the corresponding classifier $\mathcal{M}$ is justifiably fair.
\end{theorem}
Under the assumption of a reasonable classifier, Theorem~\ref{theorem: justifiable} immediately yields the following corollary, which recasts the result from the perspective of the data and forms the theoretical basis for causally fair data pre-processing.
Further discussion is provided in Appendix~\ref{appendix: diff-thm-cor}.
\begin{corollary}\label{cor: pre-justifiable-fairness}
    Let $\mathcal{G}$ be the attribute graph of a dataset. If every causal pathway in $\mathcal{G}$ that goes from sensitive attributes to label attribute includes at least one admissible attribute, then any reasonable classifier trained on this dataset satisfies justifiable fairness.
\end{corollary}

\subsection{Information Measures}

Next, we introduce key tools for measuring information~\cite{cover1999elements}, which quantify the divergence between raw data, pre-processed data, and the fair distribution in an ideal world.

The \textit{Kullback–Leibler (KL) divergence} quantifies how distribution $\mathbb{Q}$ diverges from $\mathbb{P}$:
\begin{align} \label{eq: kl}
    D_{KL}\left(\mathbb{P} \parallel \mathbb{Q}\right) 
    &= \sum_{x\in Dom(X)} \mathbb{P}[x]\log\left(\frac{\mathbb{P}[x]}{\mathbb{Q}[x]}\right),
\end{align}
where larger values indicate greater dissimilarity.

The \textit{entropy} of $X$ measures its uncertainty:
\begin{align} \label{eq: entropy}
    H(X)=-\sum_{x\in \text{Dom}(X)} \mathbb{P}[X=x] \log\mathbb{P}[X=x].
\end{align}

The \textit{mutual information (MI)} between $X$ and $Y$ captures the amount of information about $Y$ that can be learned by observing $X$:
\begin{align} \label{eq: mi}
    I(X; Y)=H(X)+H(Y)-H(XY).
\end{align}
For multiple variables $\mathcal{X}=\{X_1, \dots, X_n\}$,
\begin{align} \label{eq: mi_multi}
    I(\mathcal{X})=\sum_{\Gamma\subseteq\mathcal{X}} {(-1)^{|\Gamma|-1}H(\Gamma)},
\end{align}
which generalizes MI to quantify the total shared information.

%% file: section/3_problem.tex
\section{Problem Statement} \label{sec: problem}

Given a database instance $\mathcal{D}$ with $d$ attributes, we denote the full attribute set as $\mathcal{V} = \{V_1, \dots, V_{d-1}, Y\}$, where $Y$ is the label attribute. We further partition $\mathcal{V}$ into five disjoint subsets, $\mathcal{V} = \mathcal{S} \cup \mathcal{I} \cup \mathcal{A} \cup \mathcal{W} \cup \{Y\}$, where $\mathcal{S}$ denotes sensitive attributes (e.g., gender or race), $\mathcal{I}$ the inadmissible attributes that contain sensitive information and are thus excluded from decisions, $\mathcal{A}$ the admissible attributes whose influence is legitimate even if affected by $\mathcal{S}$, and $\mathcal{W}$ the additional attributes outside the above categories and irrelevant to sensitivity. Our objective is to design a data pre-processing framework that removes biased patterns embedded in $\mathcal{D}$. The pre-processed database, denoted by $\mathcal{D}'$, should satisfy two key properties: (i) any reasonable classifier trained on $\mathcal{D}'$ satisfies justifiable fairness as guaranteed by Corollary~\ref{cor: pre-justifiable-fairness}; and (ii) the predictive performance of downstream models trained on $\mathcal{D}'$ is effectively preserved. 

To achieve this, we modify the attribute values in individual records of $\mathcal{D}$ to calibrate its empirical distribution, ensuring that any influence from sensitive attributes on the label attribute is mediated solely through admissible attributes $\mathcal{A}$.

%% file: section/4_methodology.tex
\section{\FairRep} \label{sec: method}

As outlined in Section \ref{sec: intro}, the core idea behind \FairRep\ is to extract causally fair relationships directly from data, and this extraction task can be reformulated as a tailored distribution estimation problem. Once these fair relationships are identified, the dataset can then be processed such that its empirical distribution aligns with that of a hypothetical fair world, in which sensitive attributes influence outcomes only through admissible pathways.

At a high level, \FairRep\ operates in two main steps. In \textbf{Step-1}, it identifies causally fair relationships, denoted by the distribution $\mathbb{P}_{\mathcal{G}'}$, with respect to a fair attribute graph $\mathcal{G}'$ defined over the attribute set $\mathcal{V} = \mathcal{S} \cup \mathcal{I} \cup \mathcal{A} \cup \mathcal{W} \cup \{Y\}$. A fair attribute graph is a specialized causal DAG that includes only causally fair pathways. In \textbf{Step-2}, it processes the database $\mathcal{D}$ to produce a modified database $\mathcal{D}'$ whose empirical distribution $\mathbb{P}[\mathcal{D}']$ aligns with $\mathbb{P}_{\mathcal{G}'}$, thereby enforcing causal fairness. Note that \FairRep\ does not assume prior knowledge of the underlying DAG over attributes; all causally fair relationships in \textbf{Step-1} are inferred from scratch.

Below, we first present a naive solution (Section~\ref{subsec: naive}) that underpins \FairRep. Section~\ref{subsec: rationale} explains the design rationale. Sections \ref{subsec: distestimation} and \ref{subsec: clique} detail the core components: a tailored marginal-based distribution estimator and a heuristic for efficient marginal selection. Section \ref{subsec: datarepair} presents the complete framework, and Section \ref{subsec: tradeoff} introduces a generalized variant that balances fairness and utility.

\subsection{Naive Approach} \label{subsec: naive}

A straightforward approach to capturing the causal relationships among attributes in a database is to construct an attribute graph directly from the data. Formally, we denote this graph by $\mathcal{G}$, with attribute set $\mathcal{V} {=} \{V_1, \dots, V_{d-1}, Y\}$. For each attribute $V_i$, let $\Pi_i$ denote its parent set in $\mathcal{G}$, and let $\Pi_Y$ denote the parent set of the label $Y$. Building on this notion, the process of identifying causally fair relationships in \textbf{Step-1} can be decomposed into three phases: \textbf{(Phase-A)} construct an initial, fairness-unaware attribute graph $\mathcal{G}$ (modeled as a DAG) from the database; \textbf{(Phase-B)} transform $\mathcal{G}$ into a fair DAG $\mathcal{G}'$ by pruning edges that introduce unfairness, with updated parent sets $\Pi_i'$ and $\Pi_Y'$; and \textbf{(Phase-C)} compute causally fair relationships quantitatively by analyzing the structural properties of $\mathcal{G}'$. The ultimate goal of this three-phase process is to derive the fair joint distribution $\mathbb{P}_\mathcal{G}'$, which factorizes as:
\begin{align} \label{eq: Pg'-1}
    \mathbb{P}_{\mathcal{G}'}
    &= \mathbb{P}[V_1, V_2, \dots, V_{d-1}, Y] 
    = \prod_{i=1}^{d-1}{\mathbb{P}[V_i\mid \Pi_i']}\cdot \mathbb{P}[Y\mid \Pi_Y'],
\end{align}
where the factorization follows from Equation~(\ref{eq: bayes}).

While the factorization appears straightforward, its computation hinges on obtaining the parent sets $\Pi_i'$ and $\Pi_Y'$, which in turn depends on the initial attribute graph mentioned in \textbf{Phase-A}. A common approach is to apply causal discovery methods such as Max-Min Hill Climbing (MMHC)~\cite{TsamardinosBA06}. However, the computational complexity of MMHC grows exponentially with the number of attributes~\cite{TsamardinosBA06}, creating a significant challenge in scalability.

\subsection{Refining the Fair DAG}  \label{subsec: rationale}

While deriving an exact representation of the fair DAG $\mathcal{G}'$ in the form of Equation~(\ref{eq: Pg'-1}) offers a general recipe for achieving causal fairness, it is not strictly necessary for ensuring \emph{justifiable fairness}. In particular, the following proposition, implied by previous work~\cite{salimi2019interventional, pujol2023prefair}, establishes a more specific condition under which justifiable fairness is guaranteed.
\begin{proposition} \label{prop: fair}
    Consider a database instance $\mathcal{D}$ and its corresponding attribute graph $\mathcal{G}$ defined over a set of attributes $\mathcal{V}{=}\mathcal{S}\cup\mathcal{I}\cup\mathcal{A}\cup\mathcal{W}\cup\mathcal{Y}$. If every directed edge in $\mathcal{G}$ that points to a label attribute in $\mathcal{Y}$ originates solely from attributes in $\mathcal{A}\cup\mathcal{W}$, then any reasonable classifier trained on such a dataset is justifiably fair. Specifically, this condition is satisfied if the parent set $\Pi$ of any label attribute in $\mathcal{Y}$ is a subset of $\mathcal{A}\cup\mathcal{W}$, that is, $\Pi \subseteq \mathcal{A}\cup\mathcal{W}$.
\end{proposition}
\begin{proof}[Proof Sketch]
    The proof follows from Theorem 3.5 in~\cite{salimi2019interventional}, which builds on the definitions of $\mathcal{K}$-fairness and justifiable fairness. The complete proof is deferred to Appendix~\ref{appendix: fair}.
\end{proof}

According to Proposition~\ref{prop: fair}, a fair attribute graph can be obtained by removing all directed edges of the form $X {\rightarrow} Y$, where $X \in \mathcal{S} \cup \mathcal{I}$ and $Y \in \mathcal{Y}$, while keeping all edges between non-label attributes unchanged. This directly leads to the following refined expression for $\mathbb{P}_{\mathcal{G}'}$:
\begin{align} \label{eq: Pg'}
    \mathbb{P}_{\mathcal{G}'}
    &= \mathbb{P}[V_1, V_2, \dots, V_{d-1}, Y] 
    = \prod_{i=1}^{d-1}{\mathbb{P}[V_i\mid \Pi_i]}\cdot \mathbb{P}[Y\mid \Pi_Y'] \notag \\
    &= \mathbb{P}[\mathcal{V}\setminus\{Y\}]\cdot \mathbb{P}[Y\mid \Pi_Y\setminus(\mathcal{S}\cup\mathcal{I})]. 
\end{align}

Comparing the above expression of $\mathbb{P}_{\mathcal{G}'}$ with the original data distribution induced by the initial attribute graph,
\begin{align} \label{eq: Pg}
    \mathbb{P}_{\mathcal{G}}
    = \prod_{i=1}^{d-1}{\mathbb{P}[V_i\mid \Pi_i]}\cdot \mathbb{P}[Y\mid \Pi_Y] 
    = \mathbb{P}[\mathcal{V}\setminus\{Y\}]\cdot \mathbb{P}[Y\mid \Pi_Y],
\end{align}
we observe that only the conditional distribution of $Y$ differs, changing from $\mathbb{P}[Y\mid \Pi_Y]$ to $\mathbb{P}[Y\mid \Pi_Y\setminus(\mathcal{S}\cup\mathcal{I})]$, while the joint distribution over attributes in $\mathcal{V} \setminus \{Y\}$ remains unchanged. This insight allows us to abstract away the internal dependencies among non-label attributes and treat them collectively as a black box. Accordingly, when using the DAG as a processing guideline, it suffices to focus on the simplified structure in Figure~\ref{fig: reduce}b, where the gray block represents the black-boxed relationships, rather than the full structure in Figure~\ref{fig: reduce}a. Within this reduced view, we only need to identify the parent set of the label attribute (highlighted in the dashed orange box) and eliminate directed edges from sensitive or inadmissible attributes to the label for fairness issues.

\begin{figure}[!t]
    \centering
    \includegraphics[width=\linewidth, keepaspectratio]{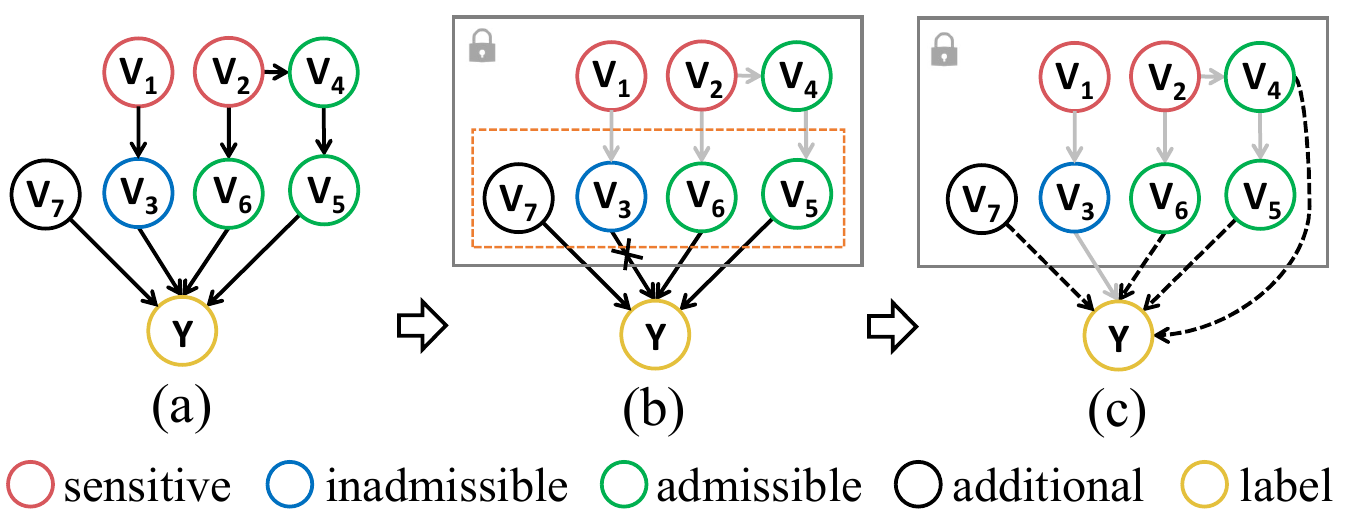}
    \caption{Refinement of the required structural guideline.}
    \label{fig: reduce}
\end{figure}

The formulation above, however, still requires access to the local structure around the label, which can also be costly to obtain, as discussed in Section~\ref{subsec: intro-motivation} and Section~\ref{subsec: naive}. To overcome this bottleneck, we present our solution, which effectively bypasses explicit DAG construction while preserving the informational utility of the database. In particular, we aim for the fair distribution $\mathbb{P}_{\mathcal{G}'}$ to remain as close as possible to the original data distribution $\mathbb{P}$. To quantify this closeness, we analyze the KL divergence between them, which in turn guides the derivation of a refined $\mathbb{P}_{\mathcal{G}'}$ better aligned with $\mathbb{P}$. By applying Equations~(\ref{eq: kl}), (\ref{eq: entropy}) and (\ref{eq: mi}), we obtain the following expression for the KL divergence from $\mathbb{P}$ to $\mathbb{P}_{\mathcal{G}'}$: 
\begin{align} \label{eq: kl3}
    D_{KL}(\mathbb{P}\parallel\mathbb{P}_{\mathcal{G}'})
    &= - \left(\sum_{i=1}^{d-1} I(V_i; \Pi_i) + I(Y; \Pi'_Y) \right) \notag \\
    & \quad {+ \sum_{i=1}^{d-1} H(V_i) + H(Y)} - H(\mathcal{V}),
\end{align}
where $I(\cdot;\cdot)$ and $H(\cdot)$ denote mutual information and entropy.

Since entropy terms in Equation (\ref{eq: kl3}) are independent of the attribute graph $\mathcal{G}'$, minimizing $D_{KL}(\mathbb{P}\parallel\mathbb{P}_{\mathcal{G}'})$ is equivalent to maximizing $\sum_{i=1}^{d-1} I(V_i; \Pi_i)+I(Y;\Pi'_Y)$. In other words, to ensure that $\mathbb{P}_{\mathcal{G}'}$ adequately represents $\mathbb{P}$, the attribute graph $\mathcal{G}'$ should maximize the mutual information between each attribute $V_i$ {(resp. $Y$)} and its parent set $\Pi_i$ {(resp. $\Pi'_Y$)}. According to the chain rule and non-negativity of conditional mutual information~\cite{cover1999elements}, for any attribute $X \in \mathcal{V}\setminus (\Pi_i\cup\{V_i\})$, we have $I\left(V_i; \Pi_i\cup\{X\}\right) \geq I\left(V_i; \Pi_i\right)+I\left(V_i; X \mid \Pi_i \right)\geq I\left(V_i; \Pi_i\right)$. Thus, enlarging a parent set cannot reduce mutual information or worsen the KL objective (Equation~\ref{eq: kl3}).

Building on this insight, we further refine the computation of $\mathbb{P}_{\mathcal{G}'}$ by greedily selecting all attributes in $\mathcal{V} \setminus (\mathcal{S} \cup \mathcal{I} \cup {Y})$ as parents of $Y$. This refinement serves as an approximation that reduces the required structural guidance from Figure~\ref{fig: reduce}b to the simpler Figure~\ref{fig: reduce}c, where the dashed orange box is no longer needed. Although this approximation may introduce additional edges directed to $Y$, the resulting factors are admissible by definition and thus acceptable for decision making. Empirically, we observe that the number of such additional edges is typically small, since most attributes in real-world datasets already act as decision-relevant factors.

Taken together, our refinement analysis progressively simplifies the required structural guidance and, as a result, the process effectively bypasses the need for explicit graph construction in \textbf{Phase-A} and \textbf{Phase-B} (see Section~\ref{subsec: naive}), while still enabling \textbf{Phase-C} to extract essential relationships. In turn, the task of causally fair relationship extraction (i.e., \textbf{Step{-}1}) reduces to estimating two distributions: $\mathbb{P}[\mathcal{V}\setminus\{Y\}]$ and $\mathbb{P}[Y\mid \Pi''_Y]$, where $\Pi''_Y=\mathcal{V}\setminus (\mathcal{S}\cup\mathcal{I}\cup\{Y\})$. We elaborate on the distribution estimation process in the following subsections.

\vspace{2mm} \noindent\textbf{Discussion on validity.} 
\FairRep\ performs fairness-aware pre-processing by referencing the refined distribution $\mathbb{P}_{\mathcal{G}'}$, in which the label attribute is causally influenced by sensitive attributes only through admissible and additional attributes. Therefore, by Corollary~\ref{cor: pre-justifiable-fairness}, any reasonable classifier trained on data following the distribution $\mathbb{P}_{\mathcal{G}'}$ is guaranteed to satisfy justifiable fairness. In addition, all preserved relationships are valid: most are untouched and directly inherited from the observational data, whereas the additional dependencies introduced around the label attribute involve only admissible factors and are, therefore, valid for downstream learning.

\subsection{Estimating Attribute Distribution} \label{subsec: distestimation}

We first present our estimation method for the attribute distribution $\mathbb{P}[\mathcal{V}\setminus\{Y\}]$. The estimation of $\mathbb{P}[Y\mid \Pi''_Y]$ is deferred to Section \ref{subsec: datarepair}. A straightforward approach to estimate $\mathbb{P}[\mathcal{V}\setminus\{Y\}]$ is to compute empirical statistics directly from $\mathcal{D}$, but this can be prohibitively expensive in time and space, especially with many attributes or large domains. For example, if $\mathcal{D}$ has 20 attributes each of size 10, the joint domain size becomes $10^{20}$, making direct estimation infeasible.

\begin{figure}[!t]
    \centering
    \includegraphics[width=\linewidth, keepaspectratio]{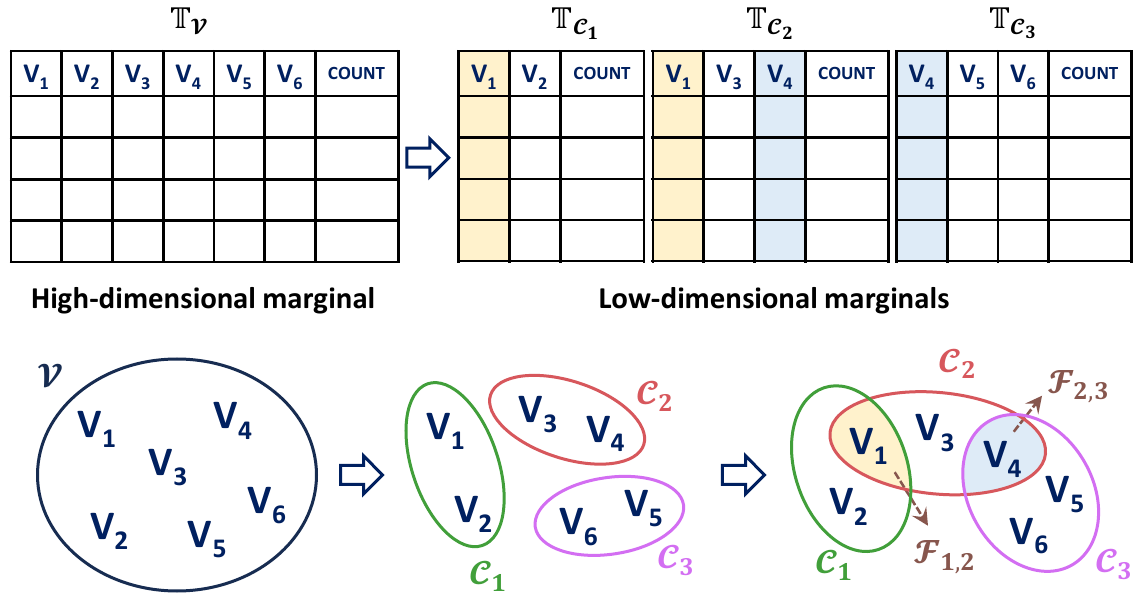}
    \caption{Illustration of the marginal-based decomposition.}
    \label{fig: high2low}
\end{figure}

To ensure scalability, we propose a marginal-based approximation method for estimating the complex joint attribute distribution $\mathbb{P}[\mathcal{V}\setminus\{Y\}]$ from the input data. The core idea is to carefully decompose the high-dimensional space into $r$ subspaces under specific constraints, each with an expected dimensionality of $(k{+}m)$, where $r$, $k$, and $m$ are predefined memory-aware parameters. Since each marginal involves significantly fewer variables than the full joint distribution, these lower-dimensional distributions can be estimated efficiently. We then combine these marginals to approximate the full joint distribution, following the factorization principle and Markov property of the junction tree~\cite{pearl2009causality, koller2009probabilistic}. To formalize this decomposition, we introduce the notion of \textit{separator}, a set of shared attributes that capture the overlap between subspaces. Given a topologically ordered set of attribute subsets $\mathscr{C}{=}\{\mathcal{C}_1, \dots, \mathcal{C}_r\}$ and the corresponding set of separator $\mathscr{F}{=}\{\mathcal{F}_{1,2}, \dots, \mathcal{F}_{r-1,r}\}$, the joint distribution over $\mathcal{V}\setminus\{Y\}$ can be decomposed and approximated as:
\begin{align} \label{eq: high2low}
    \mathbb{P}\left[\mathcal{V}\setminus\{Y\}\right] \approx \mathbb{P}\left[\mathcal{C}_1\right] \cdot \prod_{i=2}^{r} \mathbb{P}\left[\mathcal{C}_i\setminus\mathcal{F}_{i-1,i}\mid\mathcal{F}_{i-1,i}\right].
\end{align}

The upper part of Figure~\ref{fig: high2low} illustrates this decomposition. Here, the joint distribution over attributes $\mathcal{V}{=}\{V_1,\dots,V_6\}$, denoted $\mathbb{T}_{\mathcal{V}}$, is factorized into three marginal distributions $\mathbb{T}_{\mathcal{C}_1}$, $\mathbb{T}_{\mathcal{C}_2}$, and $\mathbb{T}_{\mathcal{C}_3}$. Based on Equation~(\ref{eq: high2low}), the joint distribution can be approximated as
\[\mathbb{P}\left[\mathcal{V}\right] \approx \mathbb{P}\left[V_1, V_2\right] \cdot \mathbb{P}\left[V_3,V_4\mid V_1\right] \cdot \mathbb{P}\left[V_5,V_6\mid V_4\right].\]

A crucial requirement for accurately approximating the joint distribution is to identify an appropriate attribute set for each marginal. Ideally, attributes grouped within the same marginal should have strong relationships so that the decomposition preserves key dependencies in the original data. To facilitate this, we adopt a clustering-based approach that organizes attributes into coherent groups, thereby guiding marginal selection. To ensure that the clustering effectively captures meaningful dependencies, one may use multivariate mutual information (MI), as defined in Equation~(\ref{eq: mi_multi}), as the metric to quantify the amount of shared information among attributes.

However, a naive implementation of the clustering approach using multivariate MI is computationally infeasible, as it requires comparing all combinations of attributes. Specifically, finding the optimal marginal partition involves computing entropy for all possible subsets of the $(d{-}1)$ attributes, resulting in $O(2^d)$ combinations. Moreover, computing entropy for high-dimensional subsets is computationally intensive. To tackle this issue, we adopt an efficient approximation~\cite{hall1999correlation} to estimate the total relational information among a set of attributes by summing the pairwise MI values within the set. This approximation significantly reduces the storage and computational requirements from $O(2^d)$ to $O(d^2)$, as it requires computing only $\binom{d-1}{2}$ pairwise MI values. The pairwise MI between two attributes is computed based on Equation~(\ref{eq: mi}). The clustering problem can then be modeled as a weighted complete undirected graph\footnote{Note that this undirected graph is fundamentally distinct from the fair attribute graph discussed earlier. Here, we focus on addressing the distribution estimation task introduced at the end of Section~\ref{subsec: rationale}.}, where each node represents an attribute and each edge is weighted by the pairwise MI between two attributes. The total relation information of any set of attributes can then be approximated by the sum of all edge weights within the corresponding subgraph. 

Therefore, we pose the distribution estimation problem as follows. To achieve a reliable approximation of the joint distribution, it is crucial to maximize the total relational information captured across all marginals. Formally, consider a weighted complete undirected graph $\mathcal{G}_U{=}(\mathcal{X}, \mathcal{E}, w)$, where $\mathcal{X}{=}\{X_1,\dots, X_{d-1}\}$ denotes the set of nodes (attributes), $\mathcal{E}{=}\{\langle X_i, X_j\rangle \mid X_i, X_j \in \mathcal{X}, i<j\}$ denotes the set of edges, and $w: \mathcal{E}\mapsto\mathbb{R}^+$ is a non-negative weight function defined over edges. A clique in $\mathcal{G}$ is a subset of nodes $\mathcal{C}\subseteq\mathcal{X}$ such that the induced subgraph on $\mathcal{C}$ is complete. Let $\mathcal{E}_{j,k}$ denote the weight of edge $\langle X_j, X_k\rangle$. The weight of clique $\mathcal{C}_i$ is then defined as $\sum_{X_j,X_k\in\mathcal{C}_i, \; j<k}\mathcal{E}_{j,k}$. Our goal is to determine a partitioning strategy that selects a collection of $r$ cliques $\mathscr{C} {=} \{\mathcal{C}_1, \dots, \mathcal{C}_r\}$ that maximizes the total clique weight:
\begin{align}
    \max_{\mathscr{C}}\sum^r_{i=1}\sum_{\substack{X_j,X_k\in\mathcal{C}_i \\ j<k}}\mathcal{E}_{j,k}, \label{eq: opt}
\end{align}
subject to the following conditions: 
\begin{enumerate}[leftmargin=*]
    \item Size constraint:
    \[\mid\mathcal{C}_i\mid\leq k{+}m, \forall i\in\{1, \dots, r\}.\]
    \item Coverage constraint: \[\bigcup^{r}_{i=1}{\mathcal{C}_i}=\mathcal{X}.\]
    \item Overlap constraint, i.e., each clique must overlap with at least one other clique, and every such pairwise overlap contains at least $m$ nodes: \[\forall i\in\{1, \dots, r\}, \exists j\neq i \; \text{such that}\  \vert \mathcal{C}_i \cap \mathcal{C}_j\vert \geq m.\] \[\forall i<j:\;\;|\mathcal{C}_i\cap\mathcal{C}_j|\ge m \;\text{or}\; \mathcal{C}_i\cap\mathcal{C}_j=\varnothing.\]
    \item Acyclicity constraint, i.e., the overlap relationships must be acyclic: \[\forall t\geq3, \forall(i_i,i_2,\dots,i_t) \text{ distinct},\] \[\left(\forall j\in\{1, \dots, t-1\}, \mathcal{C}_{i_j}\cap\mathcal{C}_{i_{j+1}}\neq\varnothing\right) \Rightarrow \mathcal{C}_{i_1}\cap\mathcal{C}_{i_t}=\varnothing.\]
\end{enumerate}

The above constraints aim to balance computational efficiency with approximation accuracy in modeling the joint distribution: (a)~the size of each clique is bounded to ensure that the corresponding marginal distributions remain low-dimensional, which is essential for both computational and memory efficiency;(b)~to ensure completeness, the set of resulting cliques must collectively cover all attributes; (c)~overlap and acyclicity constraints ensure that the joint distribution can be effectively approximated using Equation (\ref{eq: high2low}). 

The resulting cliques define a set of low-dimensional marginal distributions, each over the attributes of a single clique. This construction serves as a distributional proxy for the underlying causally fair dependencies, as supported by Proposition~\ref{prop: fair} and our analysis in Section~\ref{subsec: rationale}, under standard Markov compatibility and faithfulness assumptions.

\subsection{Accelerating Clique Selection} \label{subsec: clique}

While the clique-based formulation reduces the problem to tractable low-dimensional marginals, the size of the search space, together with the structural constraints, still poses significant computational challenges. The following theorem demonstrates the computational intractability of the constrained clique selection problem.
\begin{theorem}\label{thm:np-hardness}
    The constrained clique selection problem in Equation~\eqref{eq: opt} is NP-hard.
\end{theorem}
\begin{proof}[Proof Sketch]
    We prove NP-hardness via a polynomial-time reduction from the classical \emph{Exact Cover by 3-Sets} problem, which is NP-complete~\cite{garey1979computers}. The complete proof is deferred to Appendix~\ref{appendix: np-hardness}.
\end{proof}
This complexity presents a significant obstacle to the practical deployment of the proposed attribute distribution estimation method, as solving the clique selection problem exactly is computationally infeasible. To address this, we design an efficient heuristic that identifies a set of cliques sufficient to construct high-quality marginals for accurate joint distribution estimation. The heuristic consists of two main steps: \textbf{clique initialization} and \textbf{clique extension}. In the initialization step, we select edges with strong intra-clique dependencies to partition the graph into $r$ disjoint maximal cliques. Each clique contains at most $k$ (except one with $(k{+}m)$) attributes and serves as a structural core for later expansion. In the extension step, we incorporate additional attributes based on inter-clique dependencies, forming overlapping cliques of size at most $(k{+}m)$ that better capture the underlying relationships.

Figure~\ref{fig: high2low} illustrates this procedure. Consider a dataset with attribute set $\mathcal{V}{=}\{V_1, \dots, V_6\}$ and joint distribution $\mathbb{T}_\mathcal{V}$. \FairRep\ first partitions $\mathcal{V}$ into $r{=}3$ disjoint cliques: $\mathcal{C}_1{=}\{V_1, V_2\}$, $\mathcal{C}_2{=}\{V_3, V_4\}$, and $\mathcal{C}_3{=}\{V_5, V_6\}$, each of size at most $k=2$. It then partially merges these disjoint cliques to form overlapping cliques: $\mathcal{C}_1{=}\{V_1, V_2\}$, $\mathcal{C}_2{=}\{V_1, V_3, V_4\}$, and $\mathcal{C}_3{=}\{V_4, V_5, V_6\}$, each of size at most $k{+}m{=}3$. Finally, based on the attributes in each overlapping clique, \FairRep\ decomposes the joint distribution $\mathbb{T}_\mathcal{V}$ into three low-dimensional marginals, $\mathbb{T}_{\mathcal{C}1}$, $\mathbb{T}_{\mathcal{C}2}$, and $\mathbb{T}_{\mathcal{C}_3}$. A detailed description of each step is provided below. 

\input{algo/clustering}

\vspace{2mm} \noindent\textbf{Clique initialization.}
This step partitions the graph $\mathcal{G}_U$ into $r$ disjoint complete subgraphs, each forming a maximal clique. The resulting set, denoted as $\mathscr{C}_\text{init}{=}\{\mathcal{C}_1, \mathcal{C}_2, \dots, \mathcal{C}_r\}$, must satisfy three conditions: (i) each clique contains no more than $k$ attributes, except for one that may contain up to $(k{+}m)$; (ii) all attributes are covered; and (iii) the total relational information across cliques is as high as possible. The pseudo-code of this process is shown in Algorithm~\ref{algo: cluster}.

Intuitively, nodes connected by weaker dependencies are more likely to belong to different sub-structures. Based on this insight, we iteratively identify edges with the smallest weights and use their endpoints as centroids to initialize $r$ distinct cliques (lines~\ref{line: init start}--\ref{line: init end} of Algorithm~\ref{algo: cluster}).

Next, for each remaining unassigned attribute $X_u$, we evaluate its affinity with each existing clique $\mathcal{C}_i \in \mathscr{C}_{\text{init}}$ using the metric $\Delta(X_u, \mathcal{C}_i)$, defined as:
\begin{align}
    \Delta(X_u, \mathcal{C}_i) = \frac{\sum_{X_j\in\mathcal{C}_i} \mathcal{E}_{u,j}}{\sqrt{\mid\mathcal{C}_i\mid+2\sum_{\substack{X_k\in\mathcal{C}_i\setminus\{X_j\}\\k>j}}\sum_{X_j\in\mathcal{C}_i}\mathcal{E}_{j,k}}},\label{eq: delta}
\end{align}
which is a classic principal function used in correlation-based feature selection~\cite{hall1999correlation}. 

We then greedily select the attribute-clique pair $\langle X_u, \mathcal{C}_i \rangle$ with the highest $\Delta$. If $\mathcal{C}_i$ has not yet reached its capacity, $X_u$ is added; otherwise, we proceed to the next-best pair. The procedure assigns one attribute per iteration and, after $(n{-}r)$ iterations, yields the final set $\mathscr{C}_\text{init}$ of $r$ disjoint cliques (lines \ref{line: gen start}--\ref{line: gen end} in Algorithm \ref{algo: cluster}). Notably, this approach is computationally efficient, as it updates the relational information incrementally rather than recomputing $\Delta$ from scratch at each iteration.

\input{algo/cliquegen}

\vspace{2mm} \noindent\textbf{Clique extension.} 
The goal of this step is to expand the initial \emph{disjoint} clique set $\mathscr{C}_\text{init}$ into an \emph{overlapping} set $\mathscr{C}$ that satisfies both overlap and acyclicity constraints. This is achieved by iteratively merging cliques using carefully chosen separators, while ensuring a size bound on each extended clique.

To enforce acyclicity, we extend cliques in a tree-like manner, as shown in Algorithm~\ref{algo: clique}. We first select the clique of size $(k{+}m)$ from $\mathscr{C}_{init}$ as the seed, add it to $\mathscr{C}$, and mark the remaining cliques as unprocessed, denoted by $\mathscr{U}$. To satisfy the size constraint, each separator may contain at most $m$ attributes. Accordingly, we define the averaged top-$m$ affinity between cliques $\mathcal{C}_{\text{active}}$ and $\mathcal{C}_{\text{opt}}$:
\begin{align}
    \Delta_m(\mathcal{C}_\text{active}, \mathcal{C}_\text{opt}) = \max_{\{X_1, \dots. X_m\} \subseteq \mathcal{C}_{\text{opt}}} \frac{1}{m} \sum_{k=1}^{m} \Delta(X_k, \mathcal{C}_{\text{active}}), \label{eq: delta_2}
\end{align}
which averages the $\Delta$ scores of the $m$ attributes in $\mathcal{C}_{\text{opt}}$ most related to $\mathcal{C}_{\text{active}}$.

At each iteration, we identify the pair $\langle\mathcal{C}_{\text{active}}, \mathcal{C}_{\text{opt}}\rangle$ with the largest $\Delta_m$ value and take the corresponding top-$m$ attribute nodes from $\mathcal{C}_{\text{opt}}$ as the separator $\mathcal{F}_{\text{active},\text{opt}}$. We then extend $\mathcal{C}_\text{active}$ by incorporating $\mathcal{F}_{\text{active},\text{opt}}$, thereby expanding it into a larger clique. The resulting expanded clique is marked as processed and added to $\mathscr{C}$. This iterative process continues until all cliques are processed. Note that in each iteration, exactly one unprocessed clique is extended and incorporated. After $(r{-}1)$ iterations, all cliques are processed, resulting in the final ordered, overlapping, and acyclic clique set $\mathscr{C}{=}\{\mathcal{C}_1, \dots, \mathcal{C}_{r}\}$, as shown in lines~\ref{line: separator start}--\ref{line: separator end} of Algorithm~\ref{algo: clique}.

\vspace{2mm} \noindent\textbf{Putting it all together.} 
Given the ordered clique set produced by the above algorithm, each clique defines a low-dimensional marginal distribution. Together, these marginals collectively approximate the high-dimensional joint attribute distribution, as described in Equation~(\ref{eq: high2low}).

\subsection{The Complete Framework} \label{subsec: datarepair}

Recall that causally fair relationship extraction relies on estimating two key distributions, i.e., $\mathbb{P}[\mathcal{V} \setminus \{Y\}]$ and $\mathbb{P}[Y \mid \Pi_Y'']$, as discussed in Section~\ref{subsec: rationale}. While the former can be efficiently approximated using the clique-based factorization algorithm described in Section~\ref{subsec: clique}, the latter has a fundamentally different structure and therefore requires a tailored strategy. 

To unify both within the same clique-based framework, we construct an additional clique for $\mathbb{P}[Y \mid \Pi_Y'']$. Specifically, we derive the fair parent set $\Pi_Y''$ by selecting the $(k{+}m{-}1)$ fair attributes (from $\mathcal{A}\cup\mathcal{W}$) most strongly associated with $Y$. These attributes are treated as a separator and, together with $Y$, form a new clique that captures the conditional dependency. This clique is then integrated into the existing clique set, expanding $r$ cliques to $(r{+}1)$, with $r$ corresponding separators. The procedure is given in lines~\ref{line: dist start}--\ref{line: dist end} of Algorithm~\ref{algo: repair}.

\input{algo/repair}

Let $\mathscr{C} {=} \{\mathcal{C}_1, \dots, \mathcal{C}_{r+1}\}$ denote the ordered cliques, and let $\mathscr{F} {=} \{\mathcal{F}_{1,2}, \dots, \mathcal{F}_{r,r+1}\}$ denote their separators. The full distribution over the fair attribute graph $\mathcal{G}'$ can then be approximated as:
\begin{align} \label{eq: dist-est}
    \mathbb{P}_{\mathcal{G}'}
    &= \mathbb{P}[\mathcal{V}\setminus\{Y\}]\cdot \mathbb{P}[Y\mid \Pi_Y']  \notag\\
    &\approx \mathbb{P}[\mathcal{C}_1]\cdot\prod_{i=2}^{r+1}{\mathbb{P}[\mathcal{C}_i \setminus \mathcal{F}_{i-1,i} \mid \mathcal{F}_{i-1,i}]}.
\end{align}

Building on this approximation, \textbf{Step-2} (introduced at the beginning of Section~\ref{sec: method}) enforces the empirical distribution of the calibrated database $\mathcal{D'}$ to match the target distribution in Equation~(\ref{eq: dist-est}). Importantly, it is unnecessary to directly compute the full distribution $\mathbb{P}_{\mathcal{G}'}$. Instead, the procedure operates in batches: each clique in $\mathscr{C}$ is processed sequentially, as shown in lines~\ref{line: repair start}-\ref{line: repair end} of Algorithm~\ref{algo: repair}. Specifically, the procedure starts by assigning values to $\mathcal{C}_1$ from $\mathbb{P}[\mathcal{C}_1]$. For each subsequent clique $\mathcal{C}_i$ $(i{\geq}2)$, the attributes in $\mathcal{C}_i \setminus \mathcal{F}_{i-1,i}$ are filled according to $\mathbb{P}[\mathcal{C}_i \setminus \mathcal{F}_{i-1,i} \mid \mathcal{F}_{i-1,i}]$, using the values already assigned to the separator $\mathcal{F}_{i-1,i}$. Once all $(r{+}1)$ cliques are processed, the resulting calibrated database is obtained, whose empirical distribution aligns with $\mathbb{P}_{\mathcal{G}'}$.

\subsection{Utility-Fairness Trade-off} \label{subsec: tradeoff}

\FairRep\ strictly enforces the fairness constraint by fully eliminating the influence of sensitive information. To enable a more flexible balance between utility and fairness, we introduce \FairRepAd{}, a variant of \FairRep\ that interpolates between the original and fair distributions through a tunable parameter $\alpha{\in}[0,1]$. Specifically, the reference distribution of \FairRepAd{} is defined as:
\begin{align} \label{eq: trade}
    \mathbb{P}_{\mathcal{G}'-\mathcal{G}} = \alpha\mathbb{P}_{\mathcal{G}'} + (1-\alpha)\mathbb{P}_{\mathcal{G}} 
\end{align}
By varying $\alpha$, users can control the trade-off: a higher $\alpha$ prioritizes fairness at the expense of utility. Notably, when $\alpha{=}1$, \FairRepAd\ reduces to \FairRep, fully enforcing the fairness constraint.

%% file: algo/clustering.tex
\begin{algorithm}[!t] \small
    \caption{\textsc{CliqueInitialization}} 
    \label{algo: cluster}
        \KwIn{Attribute set $\mathcal{X}=\{X_1, \ldots, X_d\}$; MI matrix $\mathcal{E}=\{ \mathcal{E}_{i,j} \mid 1 \!\leq\! i \!<\! j \!\leq\! d\}$; Maximum cluster size $k$; Maximum separator size $m$}
        \KwOut{Disjoint clique set $\mathscr{C}_\text{init}$}
        Initialize $\mathscr{C}_\text{init} \gets \varnothing$, $\mathcal{U} \gets \mathcal{X}$, $\mathcal{E}' \gets \mathcal{E}$, FLAG $\gets$ TRUE\;
        Set number of clusters $r \gets \ceil{\frac{d-m-1}{k}}$\;

        \While{$\vert\mathscr{C}_\text{init}\vert < r$}{ \label{line: init start}
            Select edge $\langle i, j \rangle$ with the smallest $\mathcal{E}_{i,j}'$, where $\mathcal{E}_{i,j}' \in \mathcal{E}'$\;
            \For{$e \in \{i, j\}$}{
                \If{$X_e \in \mathcal{U}$ \textbf{and} $\vert\mathscr{C}_\text{init}\vert < r$}{
                    $\mathscr{C}_\text{init} \gets \mathscr{C}_\text{init}\cup\{\{X_e\}\}$; $\mathcal{U} \gets \mathcal{U} \setminus \{X_e\}$\;
                }
            }
            Remove $\mathcal{E}'_{i,j}$ from $\mathcal{E}'$\;
        } \label{line: init end}

        \While{$\mathcal{U} \neq \emptyset$}{ \label{line: gen start}
            Identify $\langle X_u, \mathcal{C}_i \rangle$ pair with the largest $\Delta$ value, where $\Delta$ is calculated by Equation~(\ref{eq: delta}), $X_u \in \mathcal{U}$, $\mathcal{C}_i \in \mathscr{C}_\text{init}$, and $|\mathcal{C}_i| < k{+}m$ if FLAG else $k$\;    
            Update FLAG $\gets$ FALSE if $\vert \mathcal{C}_i \vert > k$\;
            Update $\mathscr{C}_\text{init}$ by doing $\mathcal{C}_i \gets \mathcal{C}_i \cup \{X_u\}$\;  
            Remove $X_i$ from $\mathcal{U}$\; 
        } \label{line: gen end}
        \Return{$\mathscr{C}_\text{init}$}\;
\end{algorithm}

%% file: algo/cliquegen.tex
    \begin{algorithm}[!t] \small
        \caption{\textsc{CliqueExtension}} 
        \label{algo: clique}
        \KwIn{Attribute set $\mathcal{X} = \{X_1, \ldots, X_d\}$; Disjoint clique set $\mathscr{C}_\text{init} = \{\mathcal{C}_1, \ldots, \mathcal{C}_r\}$; Maximum separator size $m$}
        \KwOut{Overlapped clique set $\mathscr{C}$}
        
        $\mathcal{C}_{\text{active}} \gets$ The clique in $\mathscr{C}_\text{init}$ with size $(k{+}m)$\;
        Initialize $\mathscr{F} \gets \varnothing$, $\mathscr{C} \gets \{\{\mathcal{C}_{\text{active}}\}\}$, $\mathscr{U} \gets \mathscr{C}_\text{init} \setminus \{\mathcal{C}_{\text{active}}\}$\;
    
        \While{$\mathscr{U} \neq \varnothing$}{ \label{line: separator start}
            Identify $\langle \mathcal{C}_{\text{active}}, \mathcal{C}_{\text{opt}} \rangle $ with the largest $\Delta_m$, where $\mathcal{C}_{\text{active}} \in \mathscr{U}$, $\mathcal{C}_{\text{opt}} \in \mathscr{C}$, and $\Delta_m$ is from Equation~(\ref{eq: delta_2})\; \label{line: neighbor}
            $\mathcal{F}_{\text{active}, \text{opt}} \gets$ top-$m$ attributes from $\mathcal{C}_{\text{opt}}$ for $\Delta_m$\; \label{line: factor}
            Update $\mathscr{C} \gets \mathscr{C} \cup \{\mathcal{C}_{\text{active}}\cup\mathcal{F}_{\text{active}, \text{opt}}\}$\;
            Remove $\mathcal{C}_\text{active}$ from $\mathscr{U}$\;
        } \label{line: separator end}
        
        \Return{$\mathscr{C}$}\;
    \end{algorithm}

%% file: algo/repair.tex
\begin{algorithm}[!t] \small
    \caption{\textsc{DataPreprocessing}}
    \label{algo: repair}
    \KwIn{
        Database $\mathcal{D}$ with attribute set $\mathcal{V} = \mathcal{S} \cup \mathcal{I} \cup \mathcal{A} \cup \mathcal{W} \cup \{Y\}$; Maximum cluster size $k$; Maximum separator size $m$
    }
    \KwOut{Processed database $\mathcal{D}'$}

    \tcc{Fair relationship extraction}
    
    Initialize $\mathcal{E} \gets \varnothing$\; \label{line: dist start}
    \For{each pair of attributes $\langle X_i, X_j \rangle \in \mathcal{V}$}{
        Compute MI $\mathcal{E}_{i,j} = I[X_i; X_j]$ and add it to $\mathcal{E}$\;
    }
    $\mathscr{C}_\text{init} \gets \textsc{CliqueInitialization}(\mathcal{V}, \mathcal{E}, k, m)$\;
    $\mathscr{C} \gets \textsc{CliqueExtension}(\mathcal{V}, \mathscr{C}_\text{init}, m)$\;
    $\mathcal{F}_Y \gets$ Select $(k{+}m{-}1)$ attributes from $\mathcal{V} \setminus (\mathcal{S} \cup \mathcal{I} \cup \{Y\})$ with strongest relationship to $Y$\;
    Update $\mathscr{C} \gets \mathscr{C} \cup \{\mathcal{F}_Y \cup \{Y\}\}$\; \label{line: dist end}

    \tcc{Fairness enforcement}

    Set $n \gets |\mathcal{D}|$ \; \label{line: repair start}
    \For{$i \gets 1$ to $r+1$}{
        \eIf{$i = 1$}{
            Initialize $\mathcal{D}'$ by sampling $n$ values for attributes in $\mathcal{C}_1$ from $\mathbb{P}[\mathcal{C}_1]$\;
        }{
            $\mathcal{C}_{i-1}, \mathcal{C}_i \gets$ the ($i{-}1$)-th, $i$-th clique in $\mathscr{C}$\;
            Identify separator $\mathcal{F}_{i-1,i} \gets \mathcal{C}_{i-1} \cap \mathcal{C}_i$\;
            Sample $n$ values for $\mathcal{C}_i \setminus \mathcal{F}_{i-1,i}$ from $\mathbb{P}[\mathcal{C}_i \mid \mathcal{F}_{i-1,i}]$, and fill the corresponding columns in $\mathcal{D}'$\; \label{line: repair end}
        }
    }

    \Return{$\mathcal{D}'$}\;
\end{algorithm}

%% file: section/5_evaluation.tex
\section{Experiments} \label{sec: eval}

In this section, we present extensive experiments to evaluate the feasibility and effectiveness of \FairRep.

\subsection{Setup}

We implement the proposed \FairRep\ framework in Python\footnote{Available at \url{https://github.com/iamzhengying/CausalPre.git}}. All experiments are conducted on a machine with two Xeon(R) Gold 6326@2.90 GHz CPUs and 256GB of DRAM.

\vspace{2mm} \noindent\textbf{Datasets.}
We use three well-established benchmark datasets in our experimental evaluations: $\mathsf{Adult}$~\cite{adultData2024}, $\mathsf{COMPAS}$~\cite{compasData2024}, and $\mathsf{Census\text{-}KDD}$~\cite{censuskddData2024}, along with several synthetic datasets generated according to the causal process described in~\cite{markakis2024press}. The statistics of these datasets are summarized in Table \ref{tab: data}.

\begin{table}[t]
    \centering
    \caption{Dataset statistics.}
    \label{tab: data}
    \begin{small}
    \begin{tabular}{lcccc}
        \hline
        Dataset         & \#Tuples      & \#Attributes          & Avg. Dom   \\
        \hline
        Adult           & 32,561        & 13                    & 13.33   \\
        COMPAS          & 6,130         & 8                     & 4    \\
        Census-KDD      & 196,130       & 28                    & 11.67   \\
        Synthetic (Section~\ref{subsec: exp-recovery})       & 50,000    & 6   & 4   \\
        Synthetic (Section~\ref{subsec: exp-scalability})     & 60,000,000    & 10 -- 70   & $/$   \\
        \hline
    \end{tabular}
    \end{small}
\end{table}

\vspace{2mm} \noindent\textbf{Measurement.}
We evaluate the quality of the pre-processed dataset using three classifiers: Logistic Regression (LR), Random Forest (RF), and Multi-Layer Perceptron (MLP). Their prediction performance serves as a proxy for data quality, as more reliable training data supports more accurate and less biased learning. Specifically, each dataset is split into training and testing sets, with only the training set pre-processed to mitigate unfair patterns, while the testing set remains unmodified to reflect real-world deployment~\cite{salimi2019interventional}. Unless otherwise specified, classifiers are trained on the pre-processed training data and evaluated on the original testing data. To ensure robustness and reduce variance, we adopt 5-fold cross-validation in all experiments.

\vspace{2mm} \noindent\textbf{Baseline.}
We evaluate \FairRep\ against state-of-the-art causally fair pre-processing methods, including Cap-MS~\cite{salimi2019interventional}, Cap-MF~\cite{salimi2019interventional} and OTClean~\cite{pirhadi2024otclean}. Cap-MS and Cap-MF are originally designed to enforce saturated CI constraints over all attributes, while our setting considers unsaturated CI constraints limited to a subset. For compatibility, we extend the functionality of Cap-MS and Cap-MF to handle this scenario. In addition, OTClean's original protocol modifies both training and testing data, which differs from our evaluation setup (train-only pre-processing). To ensure a fair comparison, we consider two variants: ``OTClean-RT'', which modifies the testing data, and ``OTClean'', which leaves the testing set unchanged. Implementation details of these adapted baselines are provided in Appendix~\ref{appendix: code_extension}.

We also report results on two reference datasets: ``Original'', the unmodified dataset that reveals inherent utility and discrimination, and ``Dropped'', where all sensitive and inadmissible attributes are removed.

\vspace{2mm} \noindent\textbf{Metrics.}
The effectiveness of causally fair data pre-processing frameworks is evaluated by two criteria: preservation of data utility and the mitigation of discrimination. \textit{Utility} is assessed by the AUC score of prediction performance, with higher values indicating better utility. \textit{Discrimination} is measured by the Ratio of Observational Discrimination (ROD)~\cite{salimi2019interventional, pirhadi2024otclean}, which quantifies how much a classifier deviates from the fairness. Let $\hat{Y}$ denote the classifier output. For values $\text{\scriptsize$\mathcal{S}$}_0, \text{\scriptsize$\mathcal{S}$}1 \in \text{Dom}(\mathcal{S})$ and $a \in \text{Dom}(\mathcal{A})$, ROD is defined as:
\begin{align}
    \text{ROD} = \max_{\text{\scriptsize$\mathcal{S}$}_0, \text{\scriptsize$\mathcal{S}$}_1 \in Dom(\mathcal{S})} \frac{1}{|Dom(\mathcal{A})|} \sum_{a \in Dom(\mathcal{A})} \overline{\text{ROD}} (\text{\scriptsize$\mathcal{S}$}_0, \text{\scriptsize$\mathcal{S}$}_1; \hat{Y} \mid a), \notag
\end{align}
where 
\begin{align*}
    \overline{\text{ROD}} (\text{\scriptsize$\mathcal{S}$}_0, \text{\scriptsize$\mathcal{S}$}_1; \hat{Y} \mid a) 
    = \frac{\mathbb{P}[\hat{Y}=1 \mid \text{\scriptsize$\mathcal{S}$}_0, a] \cdot \mathbb{P}[\hat{Y}=0 \mid \text{\scriptsize$\mathcal{S}$}_1, a]}{\mathbb{P}[\hat{Y}=0 \mid \text{\scriptsize$\mathcal{S}$}_0, a] \cdot \mathbb{P}[\hat{Y}=1 \mid \text{\scriptsize$\mathcal{S}$}_1, a]}. \notag
\end{align*}

In this paper, the ROD is expressed as the absolute value of its logarithm and normalized to $[0,1]$ for consistent interpretation. A value of 0 indicates no discrimination, while higher values signify greater discrimination.

\begin{figure}[!t]
    \centering
    \begin{subfigure}[t]{\linewidth}
        \centering
        \hspace{2mm} \includegraphics[width=0.85\linewidth]{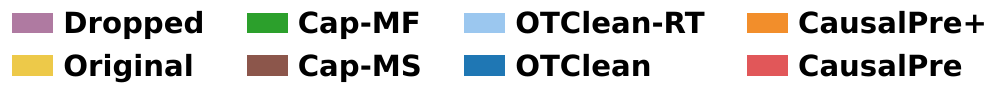}
        \vspace{2mm}
    \end{subfigure}
    \begin{subfigure}[t]{\linewidth}
        \centering
        \hspace{-2mm}\includegraphics[width=0.95\linewidth]{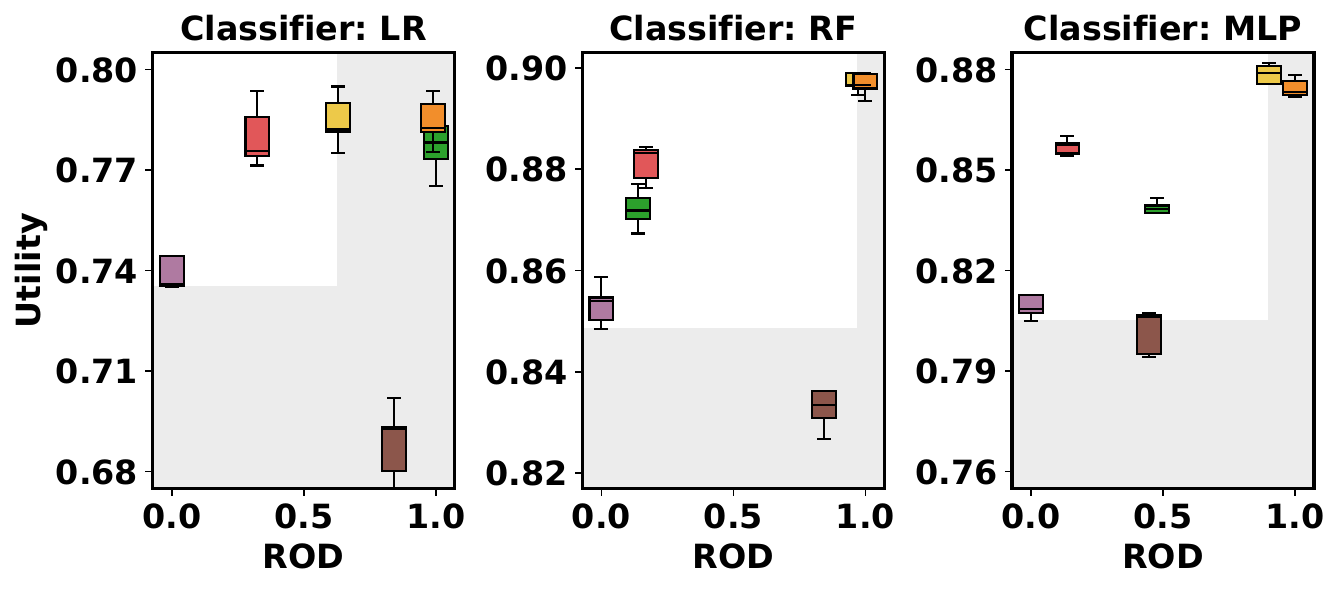}
        \vspace{-2mm}\caption{$\mathsf{Adult}$}
        \label{fig: adult}
        \vspace{2mm}
    \end{subfigure}
    \begin{subfigure}[t]{\linewidth}
        \centering
        \hspace{-2mm}\includegraphics[width=0.95\linewidth]{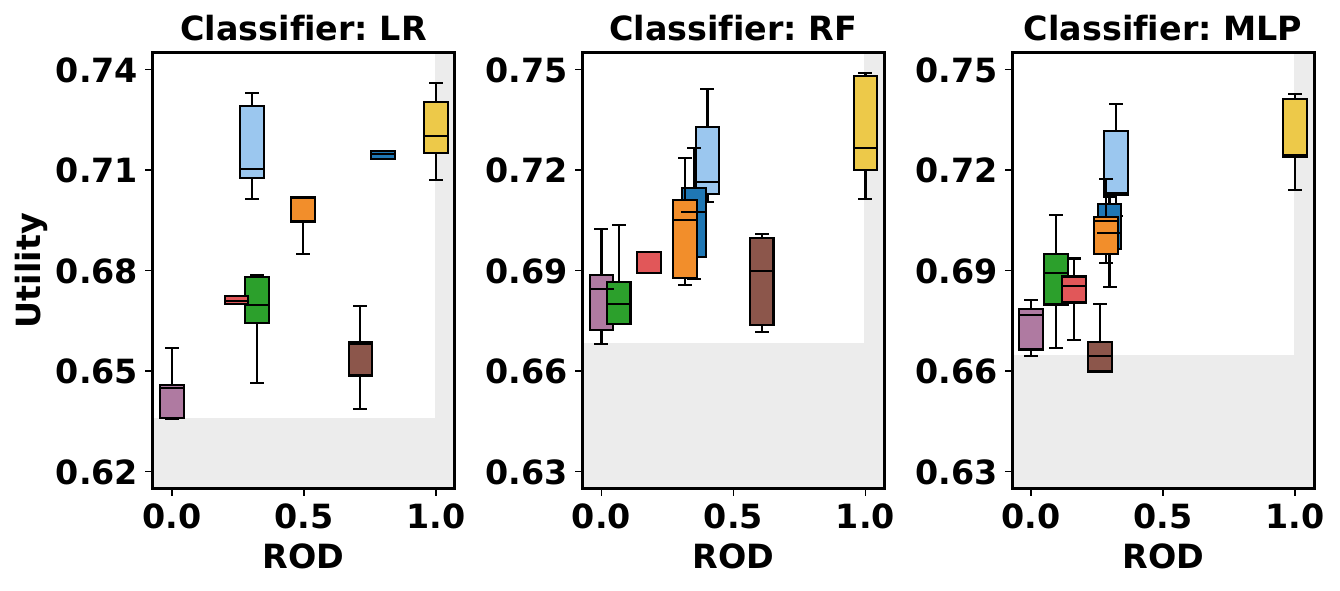}
        \vspace{-2mm}\caption{$\mathsf{COMPAS}$}
        \label{fig: compas}
        \vspace{2mm}
    \end{subfigure}
    \begin{subfigure}[t]{\linewidth}
        \centering
        \hspace{-2mm}\includegraphics[width=0.95\linewidth]{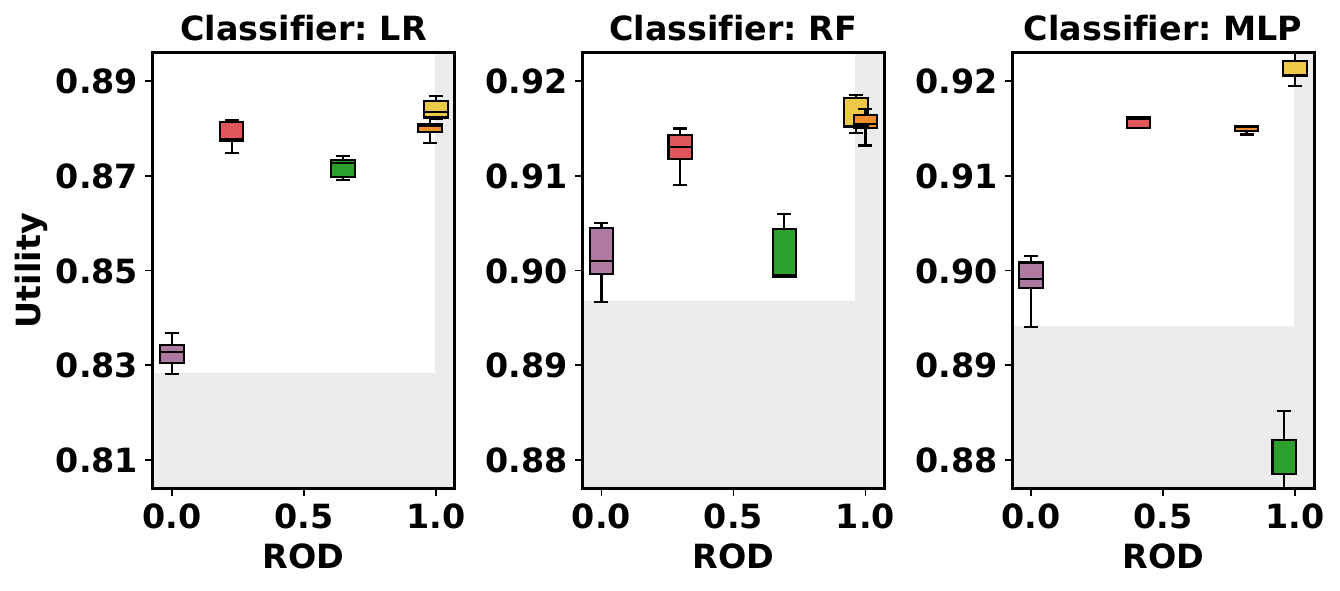}
        \vspace{-2mm}\caption{$\mathsf{Census\text{-}KDD}$}
        \label{fig: census}
    \end{subfigure}
    \caption{End-to-end performance on real-world datasets. The gray shaded area indicates invalid regions.}
    \label{fig: real}
    \vspace{-2mm}
\end{figure}

\subsection{End-To-End Performance Evaluation} \label{subsec: exp-endtoend}

We evaluate the end-to-end performance of all methods on the three benchmark datasets. The results are reported in Figure~\ref{fig: real}, where all 5-fold utility scores and average discrimination values are presented as box plots. Some baselines are excluded from specific experiments due to memory or runtime limits. For example, OTClean and OTClean-RT run out of memory on our experimental platform when dealing with datasets $\mathsf{Adult}$ and $\mathsf{Census\text{-}KDD}$, while Cap-MS fails to finish within 8 hours on $\mathsf{Census\text{-}KDD}$, even with 32 parallel threads. For \FairRep, we set ($k$, $m$) to (5, 7), (4, 3), and (6, 15) for datasets $\mathsf{Adult}$, $\mathsf{COMPAS}$, and $\mathsf{Census\text{-}KDD}$, respectively. Further parameter analysis is provided in Section~\ref{subsec: exp-param}. A method is deemed \textit{invalid} if it produces higher discrimination than ``Original'' or lower utility than ``Dropped''. Such cases are marked as shaded regions in the box plots. We highlight several key observations as follows.

First, across a broad range of datasets and classifiers, \FairRep\ is the only approach that consistently delivers valid and effective pre-processing. OTClean and OTClean-RT perform reasonably well on the smaller $\mathsf{COMPAS}$ dataset but fail to scale to larger ones. Cap-MF and Cap-MS, while principled, often produce invalid results because their strict independence constraints often over-prune the data, leading to excessive utility loss. In contrast, \FairRep\ is both robust and efficient. For example, it completes pre-processing on the largest dataset, $\mathsf{Census\text{-}KDD}$ (with around 0.2 million tuples and 28 attributes), within 2 minutes. 

Second, the variant \FairRepAd\ achieves desirable performance when labels are weakly influenced by sensitive attributes, as in $\mathsf{COMPAS}$. In this setting, $\alpha$ is set to $0.83$, $0.95$, and $0.95$ for the three classifiers. Figure~\ref{fig: alpha} illustrates the empirical utility-fairness trade-off: green points denote results for different $\alpha$ values, while the red line shows the fitted trend. Within the expected margin of variation, the overall pattern indicates that decreasing $\alpha$ relaxes the fairness constraint and improves utility. However, for $\mathsf{Adult}$ and $\mathsf{Census\text{-}KDD}$ where sensitive attributes are more strongly tied to the label, tuning $\alpha$ has a less significant effect; even high values (e.g., $\alpha{=}0.99$) may leave residual discrimination. Designing more fine-grained mechanisms to balance fairness and utility is left for future work. 

Third, among valid results, \FairRep\ consistently achieves an excellent balance between utility and fairness. On $\mathsf{Adult}$, it improves utility by an average of 1.67\% over Cap-MF while maintaining better or comparable fairness. On $\mathsf{COMPAS}$, it performs comparably to OTClean and OTClean-RT: \FairRep\ yields stronger fairness improvements, OTClean-RT achieves higher utility, and OTClean generally balances the two. Notably, OTClean-RT adjusts testing data, whereas \FairRep, \FairRepAd, and OTClean leave it unchanged, offering practical alternatives depending on the desired trade-off. On $\mathsf{Census\text{-}KDD}$, \FairRep\ reduces discrimination by 77\% with LR, 69\% with RF, and 59\% with MLP, while keeping utility drops under 0.7\%. This variation reflects model capacity: stronger models tend to extract more detailed correlations, including spurious ones, leading to different fairness-utility dynamics.

Finally, performance on small datasets like $\mathsf{COMPAS}$ shows greater variance due to limited data. Even so, \FairRep\ exhibits stronger robustness than most baselines, as evidenced by shorter box plots. This stability is attributed to its causality-guided design, which calibrates data at its causal roots.

\subsection{Parameter Analysis}\label{subsec: exp-param}

In this subsection, we study the effect of varying key parameters in \FairRep, specifically the number of cliques $r$ and the maximum clique size $(k{+}m)$. Here, $k$ and $m$ are internal parameters that control intermediate clique construction and overlap handling, respectively. The evaluation is conducted on the $\mathsf{Adult}$ dataset with the MLP classifier. 

Figure~\ref{fig: params} reports the results. In particular, Figure~\ref{fig: param1} examines the effect of varying the number of cliques $r$, with each clique set to its maximum feasible size. Figure~\ref{fig: param2} explores the impact of the maximum clique size ($k{+}m$), with the number of cliques fixed at 2, i.e., $r{=}2$. Since $r{=}\ceil{\frac{d-m-1}{k}}$ and $d{=}2$, the values of $k$ and $m$ are uniquely determined; for instance, $k{+}m{=}7$ can only result from $k{=}5$ and $m{=}2$. Overall, utility and fairness (as measured by ROD) remain relatively stable across different parameter settings. Utility shows only a slight decrease, with a 0.4\% drop when $r{>}5$ and a 2.6\% drop when $(k{+}m){<}9$. This mild degradation is likely due to the marginal loss of high-dimensional information. Interestingly, it also improves fairness by eliminating deeply embedded discriminatory signals, resulting in lower ROD scores.

\begin{figure}[!t]
    \centering
    \hspace{-3mm}
    \begin{subfigure}[t]{0.495\linewidth}
        \centering
        \includegraphics[width=\linewidth, keepaspectratio]{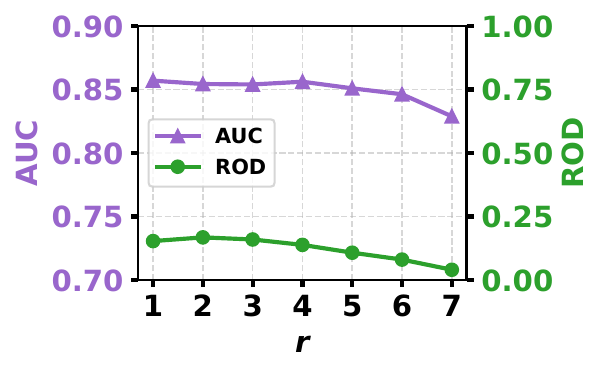}
        \vspace{-6mm}
        \caption{Varying the number of cliques.}
        \label{fig: param1}
    \end{subfigure}
    \begin{subfigure}[t]{0.495\linewidth}
        \centering
        \includegraphics[width=\linewidth, keepaspectratio]{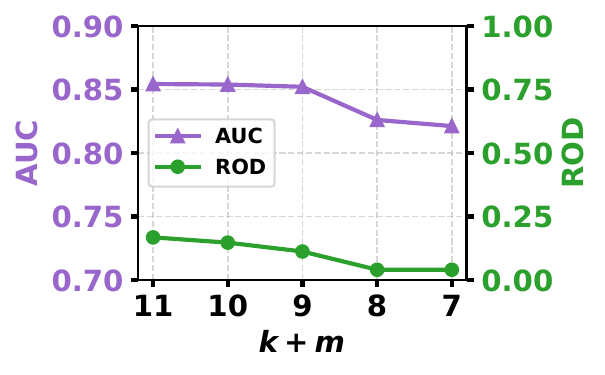}
        \vspace{-6mm}
        \caption{Varying the clique size.}
        \label{fig: param2}
    \end{subfigure}
    \caption{Parameter sensitivity analysis.}
    \label{fig: params}
\end{figure}

\begin{figure}[!t]
    \centering
    \hspace{-1.5mm}
    \includegraphics[width=0.87\linewidth, keepaspectratio]{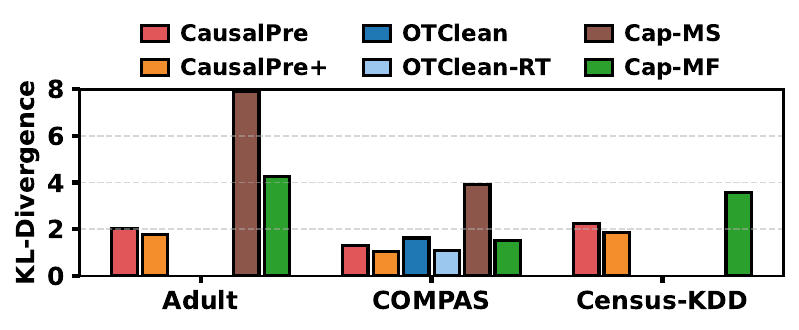}
    \vspace{-1mm} 
    \caption{Comparison of statistical distortion.}
    \label{fig: kl}
\end{figure}

\begin{figure*}[htbp]
    \centering
    \begin{subfigure}{0.13\linewidth}
      \includegraphics[width=\linewidth]{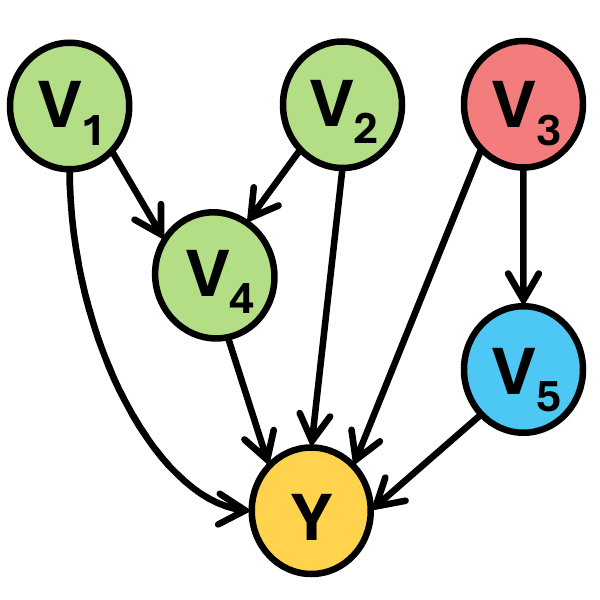}
      \caption{Ground-truth.}
      \label{fig: recovery_ori}
    \end{subfigure} \hfill
    \begin{subfigure}{0.13\linewidth}
      \includegraphics[width=\linewidth]{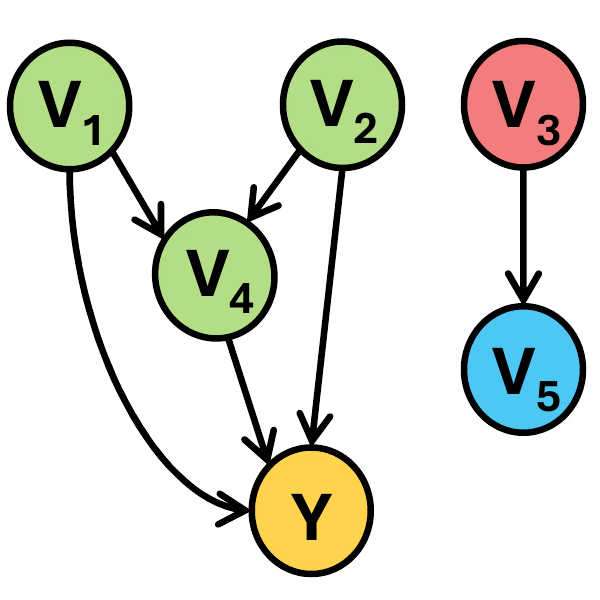}
      \caption{\FairRep.}
      \label{fig: recovery_3_lazy}
    \end{subfigure} \hfill
    \begin{subfigure}{0.13\linewidth}
      \includegraphics[width=\linewidth]{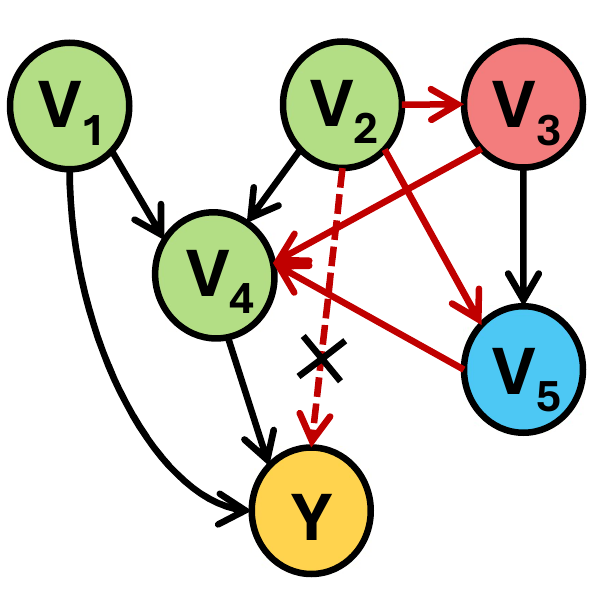}
      \caption{Cap-MF.}
      \label{fig: recovery_3_mf}
    \end{subfigure} \hfill
    \begin{subfigure}{0.13\linewidth}
      \includegraphics[width=\linewidth]{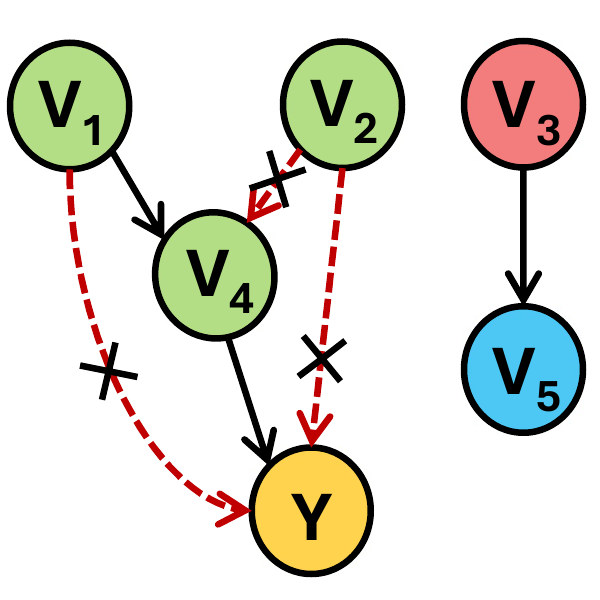}
      \caption{Cap-MS.}
      \label{fig: recovery_3_ms}
    \end{subfigure} \hfill
    \begin{subfigure}{0.13\linewidth}
      \includegraphics[width=\linewidth]{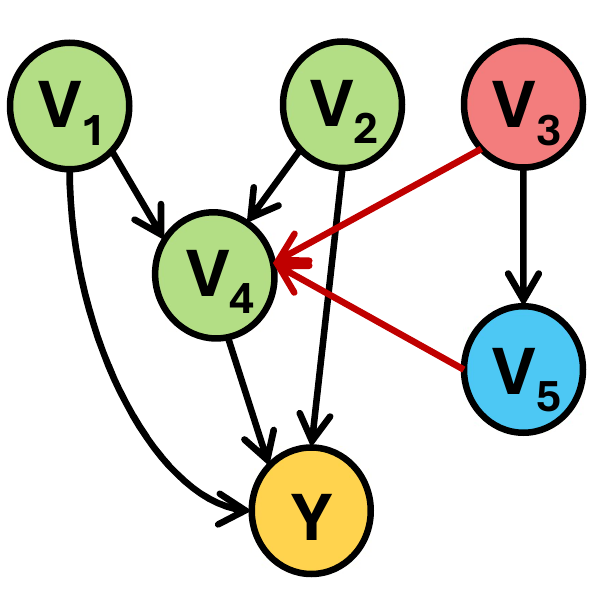}
      \caption{OTClean.}
      \label{fig: recovery_3_ot}
    \end{subfigure} \hfill
    \begin{subfigure}{0.13\linewidth}
      \includegraphics[width=\linewidth]{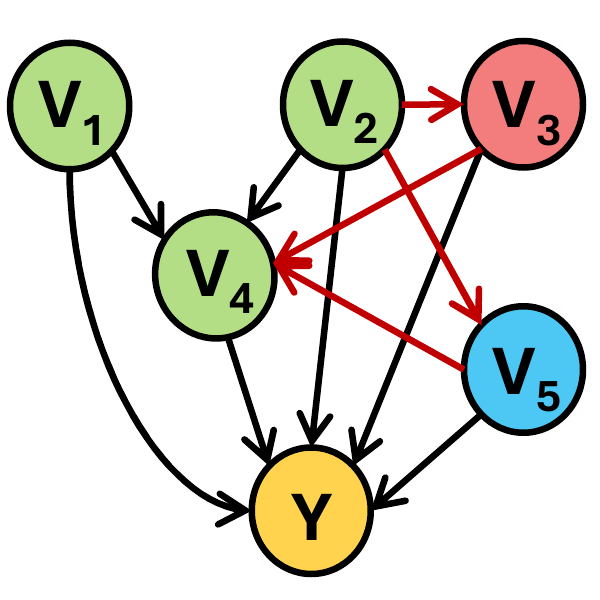}
      \caption{OTClean-RT.}
      \label{fig: recovery_3_otrt}
    \end{subfigure}
    \caption{DAG recovery: consider $V_3$ as sensitive attribute, and $V_5$ as inadmissible attribute.}
    \label{fig: recovery}
\end{figure*}

\begin{figure*}[htbp]
    \centering
    \begin{minipage}{0.265\textwidth}
        \centering
        \hspace*{-4mm}
        \includegraphics[width=\linewidth, keepaspectratio]{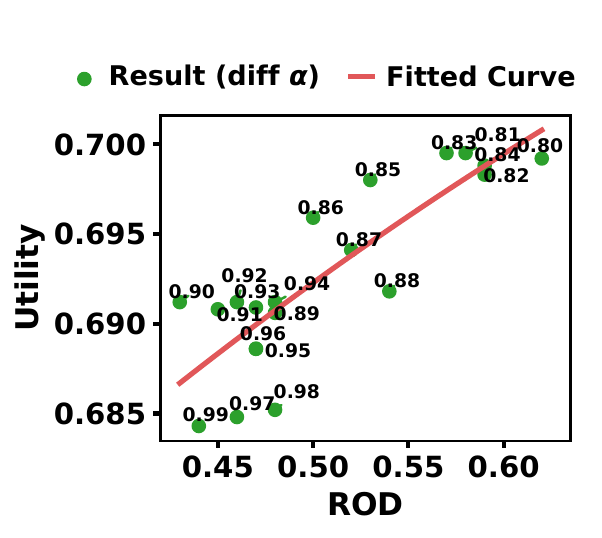}
        \caption{The utility-ROD trade-off of \FairRepAd\ on $\mathsf{COMPAS}$.}
        \label{fig: alpha}
    \end{minipage} \hfill
    \begin{minipage}{0.33\textwidth}
        \centering
        \includegraphics[width=\linewidth, keepaspectratio]{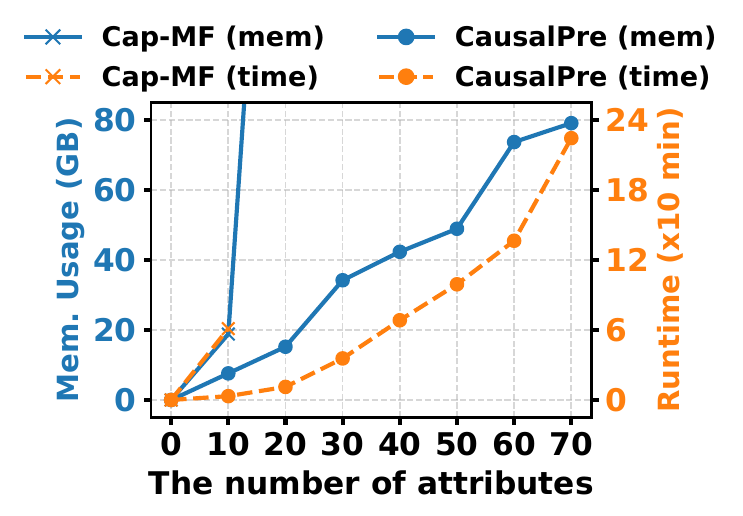}
        \caption{Memory usage and runtime under varying attribute dimensionalities.}
        \label{fig: scale_mem}
    \end{minipage} \hfill
    \begin{minipage}{0.31\textwidth}
        \centering
        \includegraphics[width=\linewidth, keepaspectratio]{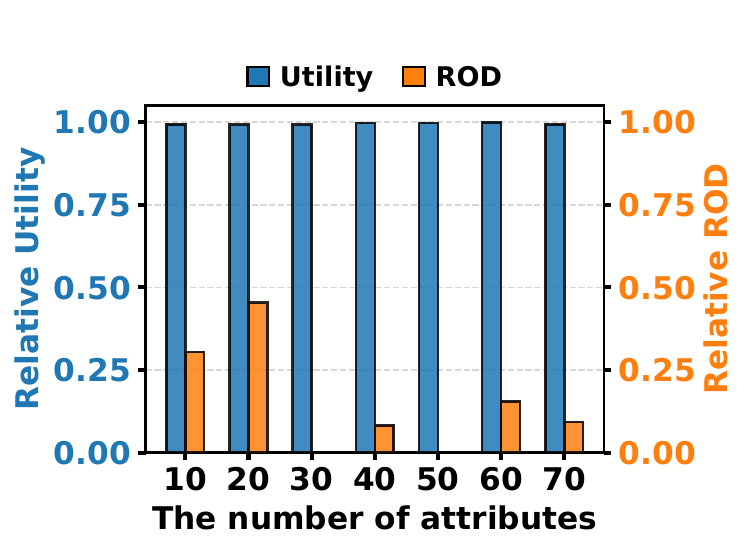}
        \caption{End-to-end performance under varying attribute dimensionalities.}
        \label{fig: scale_auc}
    \end{minipage}
\end{figure*}

\subsection{Evaluation of Statistical Distortion} \label{subsec: exp-stat}

We evaluate how well different causally fair pre-processing methods preserve the statistical structure of the original data by computing the KL-Divergence between the original and processed datasets. Smaller values indicate lower distortion and better distributional fidelity.

Figure~\ref{fig: kl} reports the results on the three real-world datasets, with some baselines excluded due to memory or runtime limits. The results reveal three key observations: (i) \FairRepAd\ consistently achieves lower KL-Divergence than \FairRep, even under a fairness-prioritized configuration with $\alpha{=}0.99$; (ii) \FairRep\ substantially outperforms Cap-MS and Cap-MF, especially on dataset $\mathsf{Adult}$, reducing KL-Divergence by over 50\% compared to Cap-MF and more than 70\% compared to Cap-MS; (iii) on smaller datasets such as $\mathsf{COMPAS}$, both \FairRepAd\ and OTClean-RT perform well. Overall, \FairRep\ demonstrates strong capability in preserving the original data distribution across diverse settings.

\subsection{Evaluation of Relationship Recovery} \label{subsec: exp-recovery}

To further measure how many original causal relationships are preserved or altered during data pre-processing, we conduct an experiment on a synthetic dataset with 6 attributes and 50,000 records, whose ground-truth attribute graph is shown in Figure~\ref{fig: recovery_ori}. In this setup, attributes are randomly assigned to roles (e.g., sensitive or inadmissible), and \FairRep\ along with baseline approaches are applied to generate processed datasets. For each dataset, we infer the underlying DAG using the Python libraries \textit{pgmpy.estimators} and \textit{networkx}, and compare the recovered structure with the ground truth to evaluate the preservation of causal relationships. Note that we keep the attribute space deliberately small to ensure that DAG recovery remains computationally feasible.

Figures~\ref{fig: recovery_3_lazy}--\ref{fig: recovery_3_otrt} present representative results, where $V_3$ is designated as sensitive, $V_5$ as inadmissible, $Y$ as the label, and the remaining variables as admissible or additional. More results under various role assignments are provided in Appendix~\ref{appendix: dag_recovery}. In the DAGs, black edges denote preserved relationships, red edges indicate spurious ones, and dashed red edges with a cross denote missing ones. OTClean-RT follows a different processing criterion, prohibiting only edges between sensitive and inadmissible attributes; its DAGs are marked accordingly. Overall, the results demonstrate that \FairRep\ is the only approach that exhibits strong fairness-aware processing: it preserves all causally fair relationships while eliminating unfair ones. In addition, it is the only method that remains valid across all cases while maintaining competitive performance in terms of both ROD and utility, as confirmed by our broader evaluation.

\subsection{Scalability Analysis} \label{subsec: exp-scalability} 

To assess scalability, we benchmark \FairRep\ and baseline methods on synthetic datasets containing 60 million records and 10--70 attributes, where each attribute has a domain size of 3--14. Figure~\ref{fig: scale_mem} presents the memory usage and runtime of \FairRep\ and Cap-MF under varying attribute dimensionalities. Other baselines could not complete even the smallest dataset and are therefore omitted from this analysis. As shown, both memory usage and runtime for \FairRep\ increase steadily with the number of attributes. In contrast, Cap-MF exhibits exponential growth in both metrics and exhausts memory limits at 20 attributes. To further assess the end-to-end performance of \FairRep\ on high-dimensional data, we evaluate its utility and ROD on these synthetic datasets, with results shown in Figure~\ref{fig: scale_auc}. Since the datasets with different numbers of attributes are independently generated, we report relative utility and relative ROD, which are computed by normalizing against the corresponding values from ``Original''. Cap-MF is excluded from this analysis as it does not scale effectively and cannot be applied to most of the evaluated cases. The results indicate that \FairRep\ preserves more than 99\% of the original utility while reducing discriminatory effects by over 50\% across all datasets through fairness-aware pre-processing.

%% file: section/6_relatedwork.tex
\section{Related Work} \label{sec: related}

Our work is most closely related to two recent efforts on causal fairness through data pre-processing: Capuchin~\cite{salimi2019interventional} and OTClean~\cite{pirhadi2024otclean}. Capuchin enforces conditional independence (CI) constraints as sufficient conditions for justifiable fairness, avoiding the need for a complete causal model. It provides two methods: Cap-MS, which reduces CI enforcement to a multi-valued dependency repair problem solved via weighted Max-SAT, and Cap-MF, which formulates pre-processing as a matrix factorization problem that approximately satisfies CI constraints. OTClean extends this idea with an optimal transport framework that integrates CI constraints while minimizing divergence between original and processed datasets, thereby improving utility preservation. Both Capuchin and OTClean, however, focus solely on CI enforcement, which can cause significant utility loss and scalability challenges, as discussed in Section~\ref{sec: intro} and demonstrated empirically in Section \ref{sec: eval}.

Another line of work includes PreFair~\cite{pujol2023prefair}, FairExp~\cite{salazar2021automated}, and SeqSel~\cite{galhotra2022causal}. PreFair also adopts causal fairness but tackles a different problem with distinct theoretical foundations. Specifically, it does not explicitly model causality; instead, it incorporates fairness constraints into the process of private data synthesis. In contrast, our work is fundamentally guided by causality. The causal intuition, along with a detailed analysis of how the problem is reformulated and why this reformulation is justified, forms the foundation of the proposed \FairRep\ framework. FairExp and SeqSel likewise address causal fairness through pre-processing but focus specifically on fair feature selection, rather than data adjustment.

Beyond causal fairness, several pre-processing methods target associational or individual fairness~\cite{feldman2015certifying, nabi2018fair, zemel2013learning, xu2018fairgan, kamiran2012data, xiong2024fairwasp, calmon2017optimized, gordaliza2019obtaining, zhang2023iflipper, lahoti2019ifair}. These approaches reweight instances, modify labels, or transform feature spaces to reduce discrimination. While effective at mitigating observed disparities, they do not account for underlying causal relationships and thus cannot guarantee causal fairness, such as justifiable fairness~\cite{salimi2019interventional}. As such, these methods are complementary but orthogonal to our work.

%% file: section/7_conclusion.tex
\section{Conclusion} \label{sec: conclu}

In this paper, we presented \FairRep, a causality-guided framework for fairness-aware data pre-processing. \FairRep\ infers causally fair relationships directly from data and uses them to guide the pre-processing steps, removing unfair influences while preserving valid dependencies among attributes. This allows the data to be reconstructed as if it were generated in a hypothetical fair world, without requiring access to a predefined causal structure. Experimental results show that \FairRep\ is the only method that consistently handles a wide range of scenarios both effectively and efficiently.

%% file: section/appendix.tex

\subsection{Difference Between Theorem~\ref{theorem: justifiable} and Corollary~\ref{cor: pre-justifiable-fairness}} \label{appendix: diff-thm-cor}

Theorem~\ref{theorem: justifiable} and Corollary~\ref{cor: pre-justifiable-fairness} differ in the objects they describe. Theorem~\ref{theorem: justifiable} characterizes classifiers, whereas Corollary~\ref{cor: pre-justifiable-fairness} characterizes training data under the assumption of a reasonable classifier. Specifically, since justifiable fairness is fundamentally a property of classifiers, Theorem~\ref{theorem: justifiable} specifies the conditions under which a classifier can be considered justifiably fair. By contrast, Corollary~\ref{cor: pre-justifiable-fairness} shifts the focus from classifiers to data, which is more directly relevant to the problem studied in this paper. Under the assumption of a reasonable classifier, namely one that can accurately learn the distribution of its training data, it characterizes what kinds of training data will yield a justifiably fair classifier.

\subsection{Proof of Proposition \ref{prop: fair}} \label{appendix: fair}

\begin{proof}
    The proof follows a line of reasoning similar to that in~\cite{salimi2019interventional}, but extends it by introducing an additional attribute category, termed ``additional'' and denoted by $\mathcal{W}$. For any superset $\mathcal{K} \supseteq \mathcal{A}$, intervening on $\mathcal{K}$, i.e., performing $\text{do}(\mathcal{K} = \text{\scriptsize$\mathcal{K}$})$, blocks all causal paths that pass through any attribute in $\mathcal{K}$. By definition, attributes in $\mathcal{W}$ are independent of the fairness constraints; therefore, no causal path exists from any sensitive attribute in $\mathcal{S}$ to any attribute in $\mathcal{W}$. Consequently, this intervention blocks all paths from any sensitive attribute in $\mathcal{S}$ to any label attribute in $\mathcal{Y}$. Therefore, intervening on $\mathcal{S}$ does not affect the distribution of $\mathcal{Y}$.
\end{proof}

\subsection{Proof of Theorem \ref{thm:np-hardness}} \label{appendix: np-hardness}

\begin{proof}
    We prove NP-hardness by constructing a polynomial-time reduction from the classical \emph{Exact Cover by 3-Sets (X3C)} problem, which is known to be NP-complete~\cite{garey1979computers}. Given an instance of X3C, we construct an instance of the constrained clique partitioning problem such that solving the latter would yield a solution to the former.

    An instance of X3C is defined as follows: let \( U {=} \{u_1, u_2, \dots, u_{3q} \} \) be a ground set of size \( 3q \), and let \( \mathcal{S} {=} \{ S_1, S_2, \dots, S_t \} \) be a collection of subsets of \( U \), where each \( S_i \subseteq U \) and \( |S_i| {=} 3 \). The goal is to determine whether there exists a subcollection \( \mathcal{S}' \subseteq \mathcal{S} \) such that \( |\mathcal{S}'| {=} q \) and every element \( u \in U \) appears in exactly one set in \( \mathcal{S}' \).

    We reduce this instance to the clique partitioning problem as follows. Let the vertex set be \( \mathcal{X} {=} U \), so that \( |\mathcal{X}| {=} 3q \). Construct a fully connected undirected graph \( \mathcal{G}_U {=} \langle \mathcal{X}, \mathcal{E}, w \rangle \), where \( \mathcal{E} {=} \{\langle X_i, X_j\rangle \mid X_i, X_j \in \mathcal{X}, i<j\} \), and the weight function \( w : \mathcal{E} \rightarrow \{0, 1\} \) is defined as:
    \[
        w(u, v) = 
        \begin{cases}
        1 & \text{if } \exists S_i \in \mathcal{S} \text{ such that } \{u, v\} \subseteq S_i, \\
        0 & \text{otherwise}.
        \end{cases}
    \]

    For each subset \( S_i \in \mathcal{S} \), define a candidate clique \( \mathcal{M}_i {=} S_i \). Set the clique size upper bound \( k + m {=} 3 \), and let the required overlap size be \( m {=} 0 \), ensuring that all cliques are disjoint. Let the target weight be \( W^* {=} 3q \), the maximum possible total weight obtainable by selecting \( q \) disjoint cliques, each of size 3 and forming a triangle of total weight 3.

    We now show that the X3C instance is a \textsc{Yes}-instance if and only if there exists a feasible clique set \( \mathscr{C} {=} \{ \mathcal{C}_{i_1}, \dots, \mathcal{C}_{i_q} \} \) in our problem satisfying all constraints and achieving total weight at least \( W^* \).

    \smallskip
    \noindent\textbf{Sufficiency (\( \Rightarrow \)).}  
    Suppose that there exists an exact cover \( \mathcal{S}' \subseteq \mathcal{S} \) with size \( |\mathcal{S}'| {=} q \). We construct \( \mathscr{C} \) by selecting the corresponding cliques \( \mathcal{C}_i = S_i \in \mathcal{S}' \). These cliques are pairwise disjoint, and their union covers all nodes in \( \mathcal{X} {=} U \), thus satisfying the coverage constraint. Each clique contributes a weight of 3, thus the total weight is \( 3q {=} W^* \). Since \( m {=} 0 \), the overlap constraint is satisfied trivially, and the overlap graph is edgeless, hence acyclic.
    
    \smallskip
    \noindent\textbf{Necessity (\( \Leftarrow \)).}  
    Suppose that there exists a clique set \( \mathscr{C} {=} \{ \mathcal{C}_{i_1}, \dots, \mathcal{C}_{i_q} \} \) satisfying the size, coverage, overlap, and acyclicity constraints, with total weight at least \( W^* {=} 3q \). Since each edge in the weight function contributes at most 1, and the maximum weight per triangle is 3, achieving total weight \( 3q \) implies that each selected clique must contain exactly 3 nodes and induce a complete subgraph of 3 edges. Moreover, the overlap constraint with \( m{=}0 \) ensures that the cliques are disjoint. Hence, the selected cliques correspond to an exact cover of \( U \).
    
    \smallskip
    The reduction is clearly computable in polynomial time. Thus, the decision version of the constrained clique partitioning problem is NP-hard. Since the optimization version generalizes this decision problem, it follows that the optimization problem is NP-hard as well.
\end{proof}


\begin{figure*}[htbp]
    \centering
    \begin{subfigure}{0.14\linewidth}
      \includegraphics[width=\linewidth]{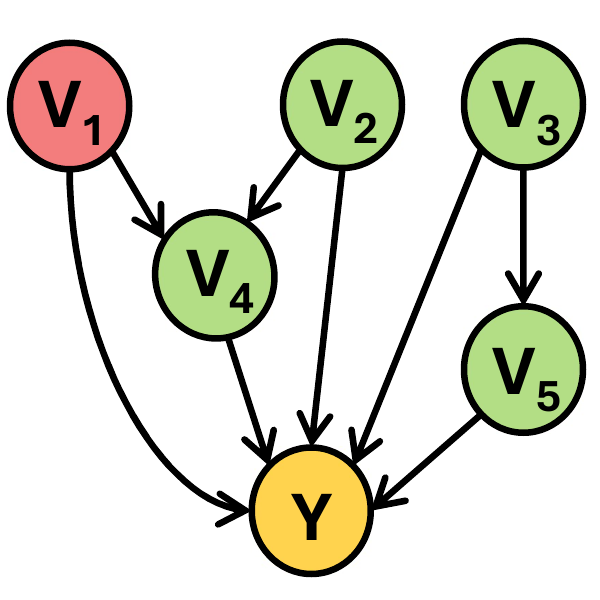}
      \caption{Ground-truth.}
      \label{fig: recovery_1_ori}
    \end{subfigure} \hfill
    \begin{subfigure}{0.14\linewidth}
      \includegraphics[width=\linewidth]{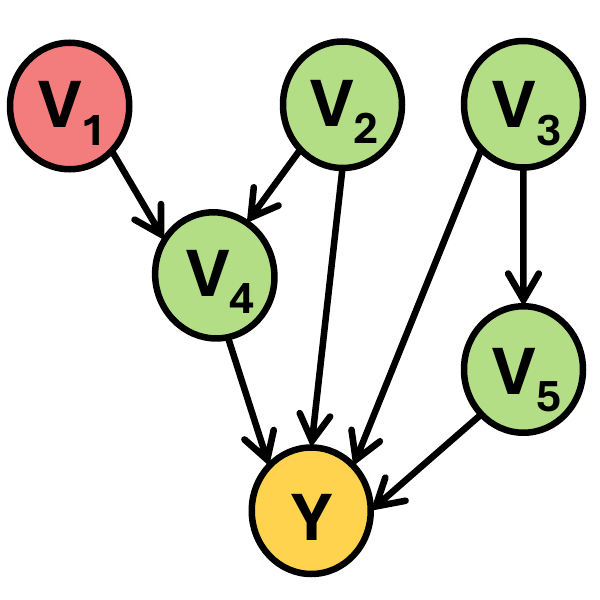}
      \caption{\FairRep.}
      \label{fig: recovery_1_lazy}
    \end{subfigure} \hfill
    \begin{subfigure}{0.14\linewidth}
      \includegraphics[width=\linewidth]{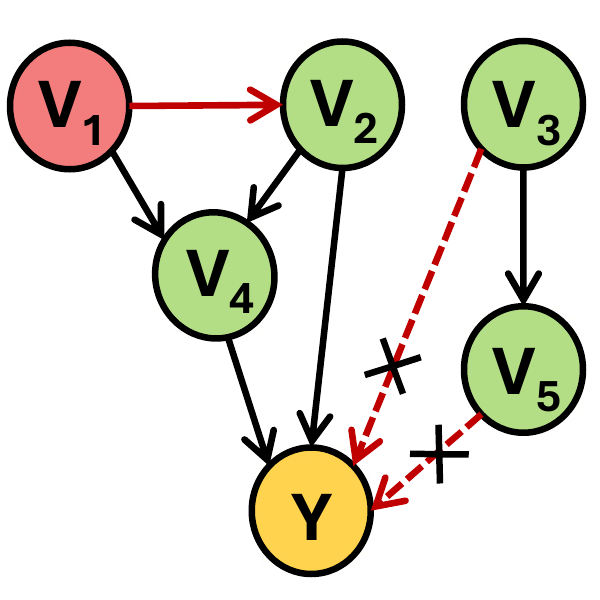}
      \caption{Cap-MF.}
      \label{fig: recovery_1_mf}
    \end{subfigure} \hfill
    \begin{subfigure}{0.14\linewidth}
      \includegraphics[width=\linewidth]{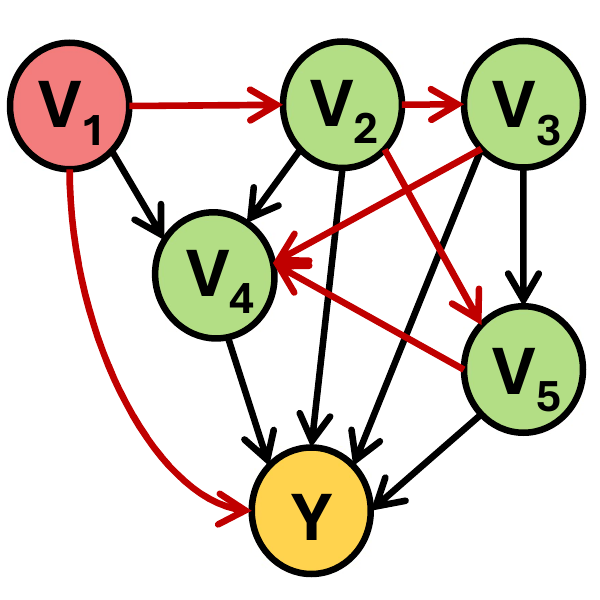}
      \caption{Cap-MS.}
      \label{fig: recovery_1_ms}
    \end{subfigure} \hfill
    \begin{subfigure}{0.14\linewidth}
      \includegraphics[width=\linewidth]{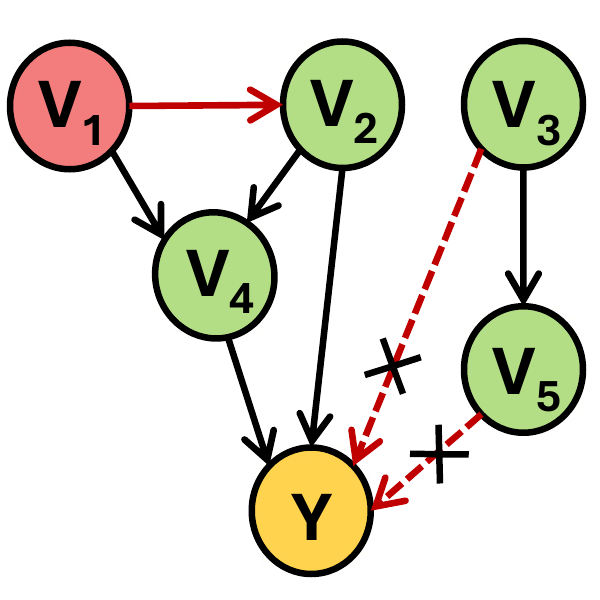}
      \caption{OTClean.}
      \label{fig: recovery_1_ot}
    \end{subfigure} \hfill
    \begin{subfigure}{0.14\linewidth}
      \includegraphics[width=\linewidth]{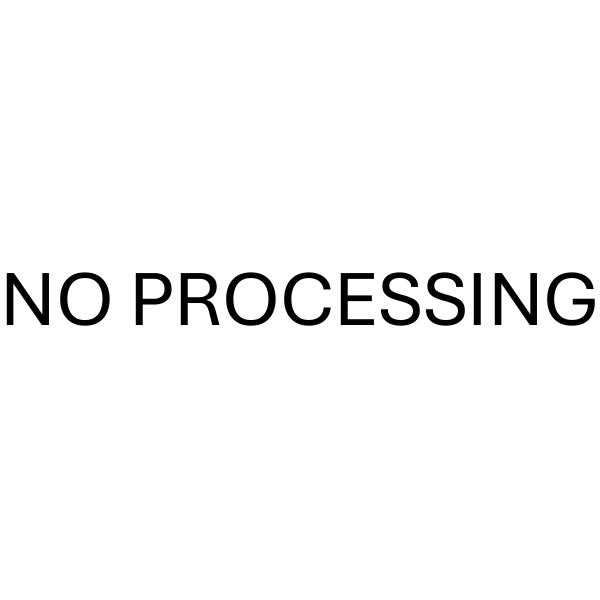}
      \caption{OTClean-RT.}
      \label{fig: recovery_1_otrt}
    \end{subfigure}
    \caption{DAG recovery: consider $V_1$ as sensitive attribute.}
    \label{fig: recovery_1}
\end{figure*}

\begin{figure*}[htbp]
    \centering
    \begin{subfigure}{0.14\linewidth}
      \includegraphics[width=\linewidth]{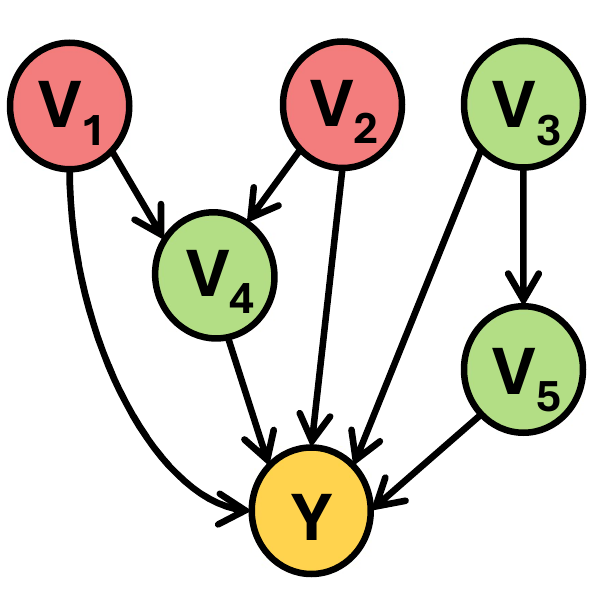}
      \caption{Ground-truth.}
      \label{fig: recovery_2_ori}
    \end{subfigure} \hfill
    \begin{subfigure}{0.14\linewidth}
      \includegraphics[width=\linewidth]{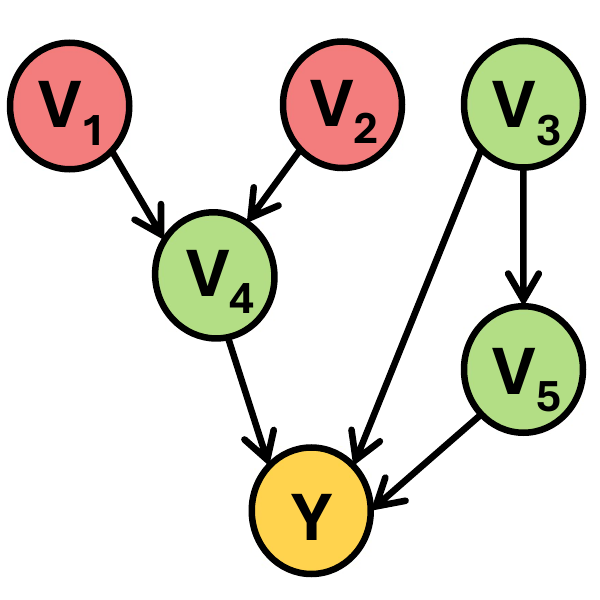}
      \caption{\FairRep.}
      \label{fig: recovery_2_lazy}
    \end{subfigure} \hfill
    \begin{subfigure}{0.14\linewidth}
      \includegraphics[width=\linewidth]{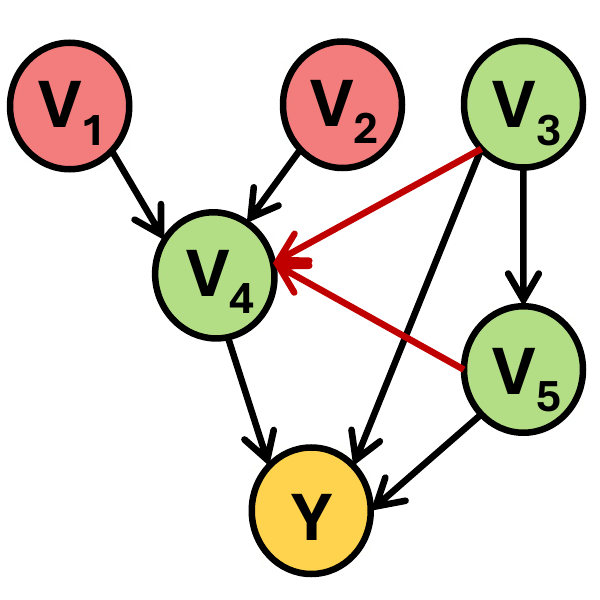}
      \caption{Cap-MF.}
      \label{fig: recovery_2_mf}
    \end{subfigure} \hfill
    \begin{subfigure}{0.14\linewidth}
      \includegraphics[width=\linewidth]{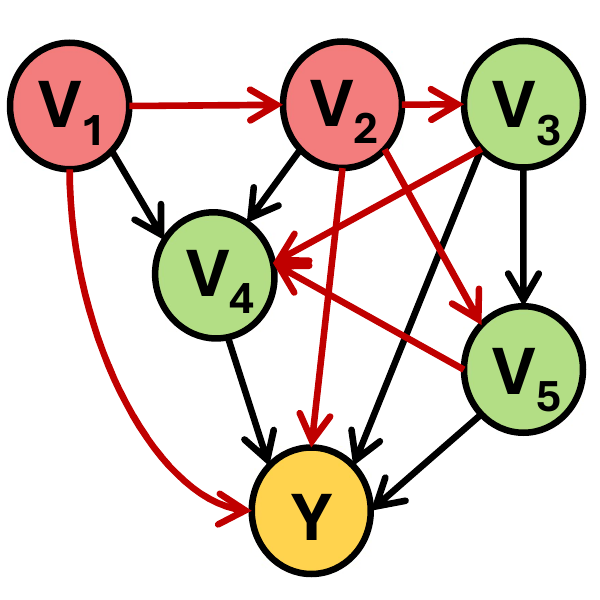}
      \caption{Cap-MS.}
      \label{fig: recovery_2_ms}
    \end{subfigure} \hfill
    \begin{subfigure}{0.14\linewidth}
      \includegraphics[width=\linewidth]{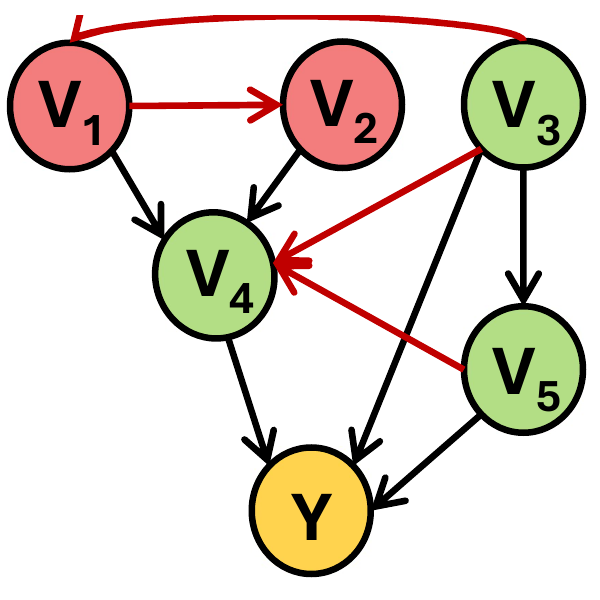}
      \caption{OTClean.}
      \label{fig: recovery_2_ot}
    \end{subfigure} \hfill
    \begin{subfigure}{0.14\linewidth}
      \includegraphics[width=\linewidth]{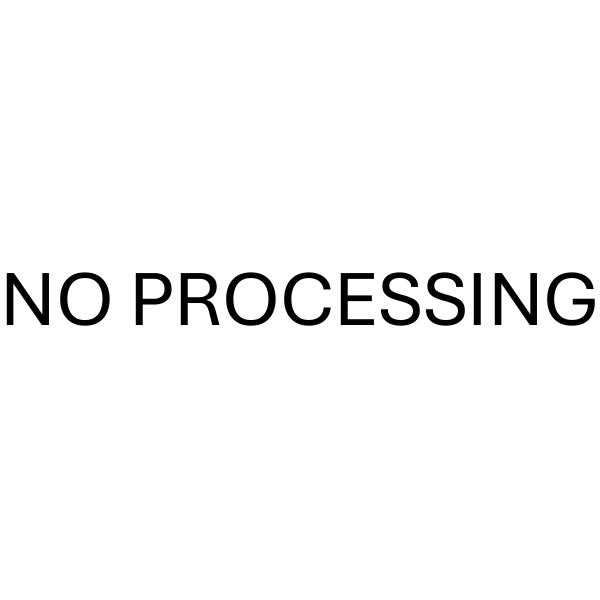}
      \caption{OTClean-RT.}
      \label{fig: recovery_2_otrt}
    \end{subfigure}
    \caption{DAG recovery: consider $V_1$ and $V_2$ as sensitive attributes.}
    \label{fig: recovery_2}
\end{figure*}

\begin{figure*}[htbp]
    \centering
    \begin{subfigure}{0.14\linewidth}
      \includegraphics[width=\linewidth]{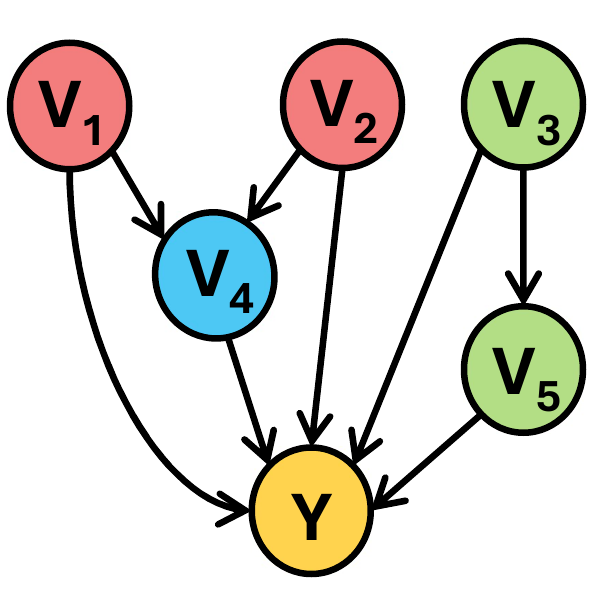}
      \caption{Ground-truth.}
      \label{fig: recovery_4_ori}
    \end{subfigure} \hfill
    \begin{subfigure}{0.14\linewidth}
      \includegraphics[width=\linewidth]{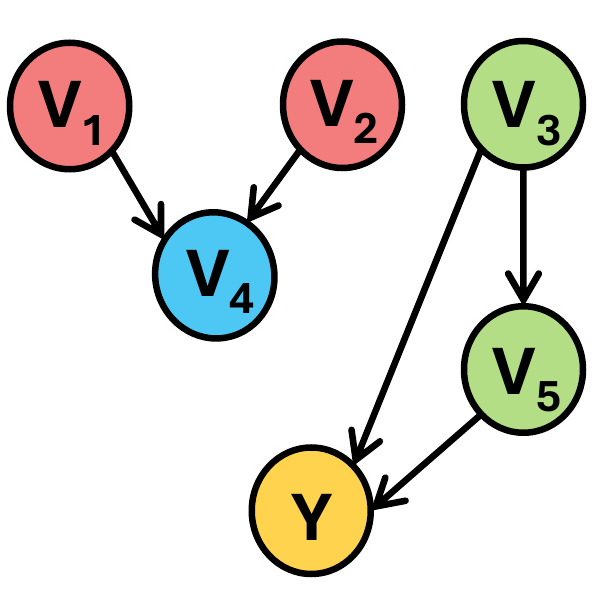}
      \caption{\FairRep.}
      \label{fig: recovery_4_lazy}
    \end{subfigure} \hfill
    \begin{subfigure}{0.14\linewidth}
      \includegraphics[width=\linewidth]{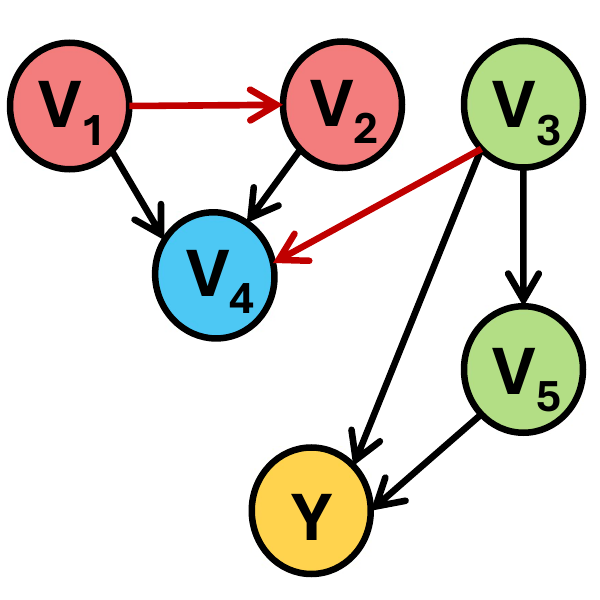}
      \caption{Cap-MF.}
      \label{fig: recovery_4_mf}
    \end{subfigure} \hfill
    \begin{subfigure}{0.14\linewidth}
      \includegraphics[width=\linewidth]{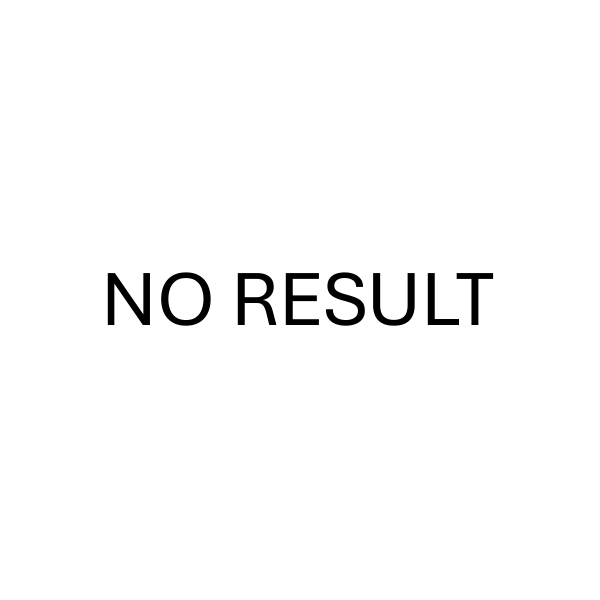}
      \caption{Cap-MS.}
      \label{fig: recovery_4_ms}
    \end{subfigure} \hfill
    \begin{subfigure}{0.14\linewidth}
      \includegraphics[width=\linewidth]{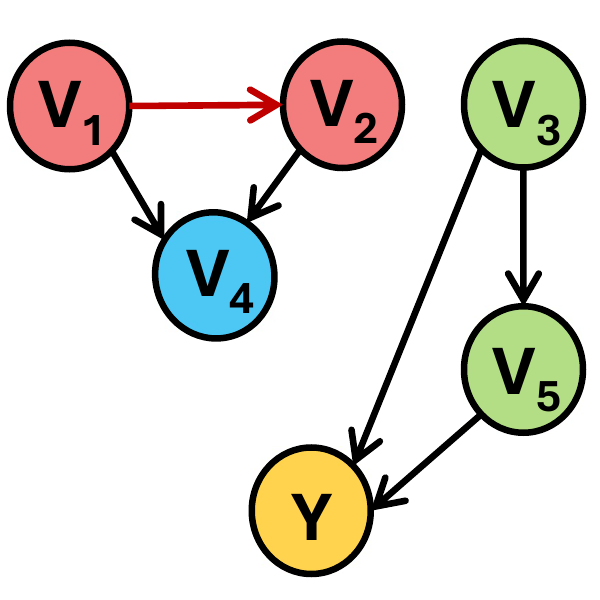}
      \caption{OTClean.}
      \label{fig: recovery_4_ot}
    \end{subfigure} \hfill
    \begin{subfigure}{0.14\linewidth}
      \includegraphics[width=\linewidth]{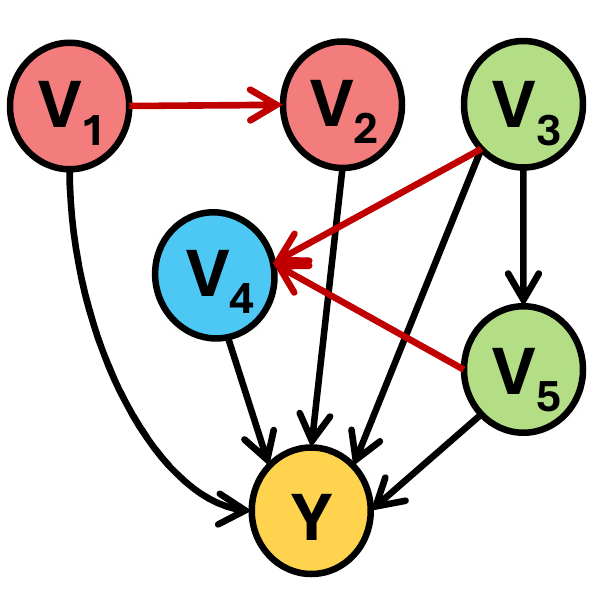}
      \caption{OTClean-RT.}
      \label{fig: recovery_4_otrt}
    \end{subfigure}
    \caption{DAG recovery: consider $V_1$ and $V_2$ as sensitive attributes, and $V_4$ as inadmissible attribute.}
    \label{fig: recovery_4}
\end{figure*}


\subsection{Details of baseline implementation} \label{appendix: code_extension}

\noindent\textbf{Cap-MS and Cap-MF.}
The original implementations of Cap-MS and Cap-MF are designed to enforce saturated conditional independence constraints, i.e., constraints that involve the full set of attributes. To adapt them for the unsaturated setting considered in this paper, we extend their functionality to support partial repair. Specifically, we first use their source code to repair the partial database $\mathcal{D}_{\mathcal{V}\setminus\mathcal{W}}$, focusing only on the attributes from $\mathcal{V}\setminus\mathcal{W}$, which involve only saturated constraints. After obtaining the repaired partial database $\mathcal{D}'_{\mathcal{V}\setminus\mathcal{W}}$, we compute the distribution of the repaired full database as $$\mathbb{P}\left[\mathcal{D}'\right]=\mathbb{P}\left[\mathcal{W}\mid\mathcal{V}\setminus\mathcal{W}\right]\cdot\mathbb{P}\left[\mathcal{D}'_{\mathcal{V}\setminus\mathcal{W}}\right],$$ where the conditional distribution $\mathbb{P}\left[\mathcal{W}\mid\mathcal{V}\setminus\mathcal{W}\right]$ is derived from the original database $\mathcal{D}$. Finally, we sample the repaired full database from the computed distribution $\mathbb{P}\left[\mathcal{D}'\right]$.

\vspace{2mm}
\noindent\textbf{OTClean and OTClean-RT.}
We include both OTClean-RT and OTClean in our experiment for fair comparison. \textit{OTClean-RT} exactly follows the original implementation from~\cite{pirhadi2024otclean}; it processes the training data to enforce the constraint $\mathcal{S} \perp \mathcal{I} \mid \mathcal{A}$, remembers modification patterns, and excludes attributes in $\mathcal{S}$ during classifier training. During prediction, it first adjusts the testing data based on the modification patterns and then makes predictions using the processed testing data. \textit{OTClean} retains the original optimal transport framework, but with the enforced constraint changed to $\mathcal{S}\mathcal{I} \perp Y \mid \mathcal{A}$ (the same setting used for Cap-MS and Cap-MF in~\cite{salimi2019interventional}). This adjustment is necessary for a fair comparison for two reasons: (i) the original constraint inherently requires testing data modification, which differs from the setting considered in this paper; and (ii) the revised constraint allows us to follow the original framework as closely as possible while keeping the testing data unmodified. With this change, the OTClean variant processes only the training data, trains on all attributes, and makes predictions on the unmodified testing data, consistent with the other methods evaluated in this paper.

\subsection{Addtional Experimental Results on Relationship Recovery} \label{appendix: dag_recovery}

Section~\ref{subsec: exp-recovery} reports results for the case with one sensitive attribute and one inadmissible attribute. For completeness, we also present results for three additional scenarios: (i) one sensitive attribute, (ii) two sensitive attributes, and (iii) two sensitive attributes with one inadmissible attribute. In each scenario, the designated attributes are randomly selected, and for simplicity, all remaining attributes (except the label) are treated as admissible. These results are shown in Figures~\ref{fig: recovery_1},~\ref{fig: recovery_2}, and~\ref{fig: recovery_4}, respectively. OTClean-RT is skipped in Figures~\ref{fig: recovery_1_otrt} and \ref{fig: recovery_2_otrt} because it applies only to scenarios where both sensitive and inadmissible attributes are present. Cap-MS is skipped in Figure~\ref{fig: recovery_4_ms} because it does not complete dataset processing within five hours.

Overall, \FairRep\ is the only approach that consistently preserves all causally fair relationships while effectively eliminating unfair ones across all scenarios. This directly demonstrates its strong ability to fix biased data without distorting valid fair relationships. In contrast, the baseline solutions frequently disrupt the original structure by aggressively eliminating legitimate causal edges, introducing spurious relationships, or even both.